\def\isarxiv{1}

\ifdefined\isarxiv

\documentclass[sigconf,authorversion]{acmart}
\setcopyright{none}
\renewcommand\footnotetextcopyrightpermission[1]{}
\else

\documentclass[sigconf,review,anonymous]{acmart}

\fi

\usepackage[utf8]{inputenc} % allow utf-8 input
\usepackage[T1]{fontenc}    % use 8-bit T1 fonts
\usepackage{hyperref}       % hyperlinks
\usepackage{url}            % simple URL typesetting
\usepackage{booktabs}       % professional-quality tables
\usepackage{amsfonts}       % blackboard math symbols
\usepackage{nicefrac}       % compact symbols for 1/2, etc.
\usepackage{microtype}      % microtypography
\usepackage{xcolor}         % colors

\usepackage{bm}
\usepackage{geometry}
\usepackage{multirow}
\usepackage{booktabs}
\usepackage{algorithm}
\usepackage{algorithmicx}
\usepackage{algpseudocode}
\usepackage{amsmath}
\usepackage{comment}
\usepackage{ulem}
\usepackage{tikz}
\usepackage{amsthm}
\usepackage{enumitem}
\usepackage{subcaption}
\usepackage[most]{tcolorbox} 

\allowdisplaybreaks[4]

\makeatletter
\def\th@acmplain{%
  \setlength{\parindent}{0pt}%
  \thm@headfont{\bfseries\scshape}% heading font (Theorem 1.1) → bold small caps
  \thm@notefont{\normalfont\scshape}% note font (convergence) → normal weight small caps
}
\makeatother

\definecolor{lightgraybg}{RGB}{235,235,235}
\newtcolorbox{graybox}{
    colback=lightgraybg,
    colframe=lightgraybg,
    fonttitle=\bfseries,
    boxsep=2pt,
    left=0pt,
    right=0pt,
    top=0pt,
    bottom=0pt,
    breakable,
    arc=0mm,
    before skip=\topsep,
    after skip=\topsep
}

\newtheorem{theorem}{Theorem}[section]

\newtheorem{proposition}[theorem]{Proposition}
\newtheorem{definition}[theorem]{Definition}
\newtheorem{remark}[theorem]{Remark}

\tcolorboxenvironment{theorem}{%
    colback=lightgraybg,
    colframe=lightgraybg,
    fonttitle=\bfseries,
    boxsep=2pt,
    left=0pt,
    right=0pt,
    top=0pt,
    bottom=0pt,
    breakable,
    arc=0mm,
    before skip=\topsep,
    after skip=\topsep
}

\tcolorboxenvironment{problem}{%
    colback=lightgraybg,
    colframe=lightgraybg,
    fonttitle=\bfseries,
    boxsep=2pt,
    left=0pt,
    right=0pt,
    top=0pt,
    bottom=0pt,
    breakable,
    arc=0mm,
    before skip=\topsep,
    after skip=\topsep
}

\tcolorboxenvironment{claim}{%
    colback=lightgraybg,
    colframe=lightgraybg,
    fonttitle=\bfseries,
    boxsep=2pt,
    left=0pt,
    right=0pt,
    top=0pt,
    bottom=0pt,
    breakable,
    arc=0mm,
    before skip=\topsep,
    after skip=\topsep
}

\tcolorboxenvironment{corollary}{%
    colback=lightgraybg,
    colframe=lightgraybg,
    fonttitle=\bfseries,
    boxsep=2pt,
    left=0pt,
    right=0pt,
    top=0pt,
    bottom=0pt,
    breakable,
    arc=0mm,
    before skip=\topsep,
    after skip=\topsep
}

\tcolorboxenvironment{lemma}{%
    colback=lightgraybg,
    colframe=lightgraybg,
    fonttitle=\bfseries,
    boxsep=2pt,
    left=0pt,
    right=0pt,
    top=0pt,
    bottom=0pt,
    breakable,
    arc=0mm,
    before skip=\topsep,
    after skip=\topsep
}

\tcolorboxenvironment{proposition}{%
    colback=lightgraybg,
    colframe=lightgraybg,
    fonttitle=\bfseries,
    boxsep=2pt,
    left=0pt,
    right=0pt,
    top=0pt,
    bottom=0pt,
    breakable,
    arc=0mm,
    before skip=\topsep,
    after skip=\topsep
}

\tcolorboxenvironment{definition}{%
    colback=lightgraybg,
    colframe=lightgraybg,
    fonttitle=\bfseries,
    boxsep=2pt,
    left=0pt,
    right=0pt,
    top=0pt,
    bottom=0pt,
    breakable,
    arc=0mm,
    before skip=\topsep,
    after skip=\topsep
}

\tcolorboxenvironment{remark}{%
    colback=lightgraybg,
    colframe=lightgraybg,
    fonttitle=\bfseries,
    boxsep=2pt,
    left=0pt,
    right=0pt,
    top=0pt,
    bottom=0pt,
    breakable,
    arc=0mm,
    before skip=\topsep,
    after skip=\topsep
}

\tcolorboxenvironment{fact}{%
    colback=lightgraybg,
    colframe=lightgraybg,
    fonttitle=\bfseries,
    boxsep=2pt,
    left=0pt,
    right=0pt,
    top=0pt,
    bottom=0pt,
    breakable,
    arc=0mm,
    before skip=\topsep,
    after skip=\topsep
}

\newcommand{\R}{\mathbb{R}}

\newcommand{\wt}[1]{\widetilde{#1}}
\newcommand{\ov}[1]{\overline{#1}}

\newcommand{\argmin}{\mathop{\arg\min}}

\newcommand{\thetab}{\boldsymbol{\theta}}

\newcommand{\xb}{\mathbf{x}}

\newcommand{\Ab}{\mathbf{A}}

\newcommand{\Db}{\mathbf{D}}

\newcommand{\Xb}{\mathbf{X}}

\newcommand{\Ccal}{\mathcal{C}}

\newcommand{\Ecal}{\mathcal{E}}

\newcommand{\Gcal}{\mathcal{G}}

\newcommand{\Lcal}{\mathcal{L}}

\newcommand{\Vcal}{\mathcal{V}}

\newcommand{\Ycal}{\mathcal{Y}}

\newcommand{\AppendixName}{Appendix}

\input{__append_toc}

\newcommand{\circledcolor}[1]{\textcolor{#1}{\raisebox{-0.5ex}{\Huge $\bullet$}}}

\definecolor{nodeinjected}{HTML}{D08268}
\definecolor{nodeunlabeled}{HTML}{EDBC81}
\definecolor{nodelabeled}{HTML}{FCF3E7}

\AtBeginDocument{%
  }

\setcopyright{acmlicensed}
\copyrightyear{2018}
\acmYear{2018}
\acmDOI{XXXXXXX.XXXXXXX}
\acmConference[Conference acronym 'XX]{Make sure to enter the correct
  conference title from your rights confirmation email}{June 03--05,
  2018}{Woodstock, NY}
\acmISBN{978-1-4503-XXXX-X/2018/06}

\begin{document}

\title{Attack by Unlearning: Unlearning-Induced Adversarial Attacks on Graph Neural Networks}
%\title{Adversarial Attacks on Graph Neural Networks via Adversarial Unlearning}


\ifdefined\isarxiv

\author{Jiahao Zhang, Yilong Wang, Suhang Wang}
\affiliation{%
  \institution{The Pennsylvania State University}
  \city{University Park, PA}
  \country{USA}
  }
\email{{jiahao.zhang, yvw5769, szw494}@psu.edu}

\else

\author{Jiahao Zhang}
\affiliation{%
  \institution{The Pennsylvania State University}
  \city{University Park}
  \country{United States}
  }
\email{jiahao.zhang@psu.edu}

\author{Yilong Wang}
\affiliation{%
  \institution{The Pennsylvania State University}
  \city{University Park}
  \country{United States}
  }
\email{yvw5769@psu.edu}

\author{Suhang Wang}
\affiliation{%
  \institution{The Pennsylvania State University}
  \city{University Park}
  \country{United States}
  }
\email{szw494@psu.edu}
\fi

\renewcommand{\shortauthors}{Zhang et al.}

\begin{abstract}
  Graph neural networks (GNNs) are widely used for learning from graph-structured data in domains such as social networks, recommender systems, and financial platforms. To comply with privacy regulations like the GDPR, CCPA, and PIPEDA, approximate graph unlearning, which aims to remove the influence of specific data points from trained models without full retraining, has become an increasingly important component of trustworthy graph learning. However, approximate unlearning often incurs subtle performance degradation, which may incur negative and unintended side effects. 
In this work, we show that such degradations can be amplified into adversarial attacks. We introduce the notion of \textbf{unlearning corruption attacks}, where an adversary injects carefully chosen nodes into the training graph and later requests their deletion. Because deletion requests are legally mandated and cannot be denied, this attack surface is both unavoidable and stealthy: the model performs normally during training, but accuracy collapses only after unlearning is applied. 
Technically, we formulate this attack as a bi-level optimization problem: to overcome the challenges of black-box unlearning and label scarcity, we approximate the unlearning process via gradient-based updates and employ a surrogate model to generate pseudo-labels for the optimization. 
Extensive experiments across benchmarks and unlearning algorithms demonstrate that small, carefully designed unlearning requests can induce significant accuracy degradation, raising urgent concerns about the robustness of GNN unlearning under real-world regulatory demands. The source code will be released upon paper acceptance. 

\end{abstract}

%\ifdefined\isarxiv
%\else
\begin{CCSXML}
<ccs2012>
   <concept>
       <concept_id>10010147.10010257</concept_id>
       <concept_desc>Computing methodologies~Machine learning</concept_desc>
       <concept_significance>500</concept_significance>
       </concept>
   <concept>
       <concept_id>10002978.10003022.10003026</concept_id>
       <concept_desc>Security and privacy~Web application security</concept_desc>
       <concept_significance>500</concept_significance>
       </concept>
   <concept>
       <concept_id>10002978.10003006.10011634</concept_id>
       <concept_desc>Security and privacy~Vulnerability management</concept_desc>
       <concept_significance>500</concept_significance>
       </concept>
 </ccs2012>
\end{CCSXML}

\ccsdesc[500]{Computing methodologies~Machine learning}
\ccsdesc[500]{Security and privacy~Web application security}
\ccsdesc[500]{Security and privacy~Vulnerability management}

\keywords{Graph Unlearning, Adversarial Attack, Graph Neural Networks}


%\fi

\maketitle

\section{Introduction}

Graph-structured data are ubiquitous in the real world, such as social networks~\cite{fan2019graph,sankar2021graph,liu2024score}, recommender systems~\cite{wang2019neural,he2020lightgcn,zhang2024linear}, financial networks~\cite{dou2020enhancing,liu2020alleviating,wang2025enhance}, knowledge bases~\cite{liu2025exposing,yang2026query,luo2026graphs}, and biological interaction networks~\cite{strokach2020fast,truong2024prediction,xu2025dualequinet}. 
Graph neural networks (GNNs) have become the de facto approach for learning from graph-structured data because they can effectively combine node features and graph topology to produce powerful predictions~\cite{kipf2017semi,velickovic2018graph,xu2019powerful}. 
At the same time, many graph datasets may contain sensitive or outdated information that must be efficiently removed from the trained GNN models~\cite{zhang2021membership, xin2023user,pedarsani2011privacy,wang2024data}. For example, purchasing records in recommender systems~\cite{chen2022recommendation,zhang2024recommendation} or loan histories in financial networks~\cite{pedarsani2011privacy,wang2024data} may reveal personally identifying or regulated information. When GNNs are trained on such data and the models are released via APIs, the models themselves can encode sensitive information and become a source of privacy and compliance risk. In addition, regulations such as the GDPR~\cite{gdpr}, CCPA~\cite{ccpa}, and PIPEDA~\cite{pipeda} further require systems to respect the user's ``right to be forgotten'', motivating a large and growing literature on graph unlearning~\cite{chien2022certified,chen2022graph,wu2023gif}, methods for removing the influence of specific data points from a trained GNN without full retraining.

\begin{figure*}[!ht]
    \centering
    \includegraphics[width=1\linewidth]{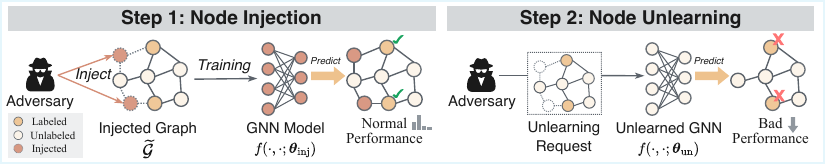}
    \vskip -0.5em
    \caption{Attack setting of the proposed unlearning corruption attack on GNNs.} 
    \vskip -0.05in
    \label{fig:attack_setting}
\end{figure*}

Despite the recent advancements in graph unlearning, concerns regarding its negative side effects are rising. One of the most well-known side effects is that unlearning is necessarily an approximation of full retraining, and this approximation can incur subtle but unavoidable performance degradation~\cite{chen2023boundary,yao2024large,zhang2024negative,wu2023gif}. In this work, we ask a timely and fundamental research question: 

\begin{graybox}
    {\it Can the unavoidable performance degradation in unlearning be amplified and weaponized to break a GNN, by carefully crafting the unlearning requests themselves?}
\end{graybox}

We refer to such manipulations as {\bf unlearning corruption attacks} (Figure~\ref{fig:attack_setting}). Concretely, an attacker first injects carefully crafted adversarial nodes (the brown circles \circledcolor{nodeinjected} in Figure~\ref{fig:attack_setting}) into the training graph, and later the attacker will submit deletion (unlearning) requests to remove those nodes from a GNN trained on the poisoned graph. The nodes are crafted in a way such that the approximate unlearning of those nodes from the target GNN would substantially degrade prediction performance on unlabeled nodes (the gold circles \circledcolor{nodeunlabeled} in Figure~\ref{fig:attack_setting}).

This new attack surface for graph unlearning is highly practical in real-world scenarios for two key reasons. 
First, the unlearning requests are \textbf{hard to reject} from the model owner's perspective: node creation and deletion are legitimate user operations, where deletion requests in particular are often legally mandated by privacy laws. Refusing such requests can carry legal and reputational costs, which constrain defenders' choices. Second, the attack is \textbf{highly stealthy}: while injected nodes may not noticeably affect model accuracy during training, performance collapse can appear only after deletion/unlearning is performed (i.e., often only when the model is already deployed). 
Furthermore, we would like to also emphasize that this attack paradigm is novel, since it uniquely exploits the unlearning process itself as the trigger mechanism, distinct from traditional poisoning or backdoor attacks. To support this claim, we provide a comprehensive comparison with existing attacks in Section~\ref{sec:formal_prob_dfn} and Appendix~\ref{sec:discussion}.

In this paper, we present a systematic study of the vulnerabilities of graph unlearning in a semi-supervised node-classification setting. We first identify the core principles governing effective unlearning corruption by formalizing three key adversarial goals: \textbf{(G1) Post-unlearning damage}, ensuring the model's accuracy collapses significantly on the target population after the unlearning request; \textbf{(G2) Pre-unlearning utility}, requiring the attack to remain latent and stealthy during the initial training phase; and \textbf{(G3) Stealthiness under benign unlearning}, ensuring that unlearning other, benign nodes does not accidentally trigger the degradation. Achieving these adversarial goals simultaneously presents several technical challenges: (i) the unlearning operator is often a ``black box'' or involves complex gradient updates~\cite{wu2023gif,wu2023certified,wu2024graphguard}, making it difficult to compute the gradients necessary to optimize the trigger for the post-unlearning state (G1). Second, the attacker typically lacks ground-truth labels for the unlabeled nodes, preventing direct optimization of the utility and stealthiness constraints (G2, G3). 

To address these challenges, we propose a novel unlearning corruption attack framework that mathematically formalizes the attacker's problem as a bi-level optimization task. We address the black-box challenge by approximating the unlearning process via gradient-based updates and address the label scarcity issue by training a surrogate model to generate pseudo-labels. This allows our framework to solve for the optimal injection of nodes and edges that maximally degrade the model upon unlearning while strictly adhering to utility and stealthiness constraints. Extensive experiments across datasets and unlearning algorithms demonstrate that these carefully designed unlearning requests can induce significant accuracy drop, raising urgent concerns about the robustness of graph unlearning under real-world regulatory demands.

{\bf Contributions.} Our work makes the following contributions:
\begin{itemize}[leftmargin=0.05\linewidth]
    \item We introduce the notion of {\bf unlearning corruption attacks} against GNNs and formalize a threat model where adversaries inject nodes and later submit unlearning requests that the model owner must follow due to privacy regulations.
    \item We propose a novel unlearning corruption attack framework that simultaneously addresses all three attack goals (G1-G3). Our framework optimizes the trade-off between maintaining pre-unlearning utility and maximizing post-unlearning damage, while ensuring performance remains unaffected under benign unlearning requests.
    \item We demonstrate the effectiveness of our attacks on multiple GNN models, datasets, and graph unlearning algorithms, and discuss practical defenses and the tension between compliance and robustness.
\end{itemize}

\section{Related Works} \label{sec:rel_works}

{\bf Graph Unlearning.} 
Graph unlearning aims to efficiently remove the impact of undesired data from trained graph learning models~\cite{chen2022graph,said2023survey,fan2025opengu}. 
The main challenge lies in balancing model utility, 
efficiency, and removal guarantees, which have led to two primary research directions: retrain-based and approximate unlearning. 

Retrain-based methods divide the training graph into disjoint subgraphs and train separate submodels, enabling unlearning by retraining only the affected partitions. 
GraphEraser~\cite{chen2022graph} introduced the first retraining-based GNN unlearning framework, leveraging balanced graph clustering for partitioning and a weighted ensemble of submodels with trainable fusion weights to retain high utility. 
GUIDE~\cite{wang2023inductive} further extends this paradigm to inductive graph learning scenarios. 
Follow-up studies~\cite{wang2023inductive,li2023ultrare,zhang2024graph,li2025community} have advanced the trade-off between utility and efficiency using techniques like data condensation~\cite{li2025community} and improved clustering strategies~\cite{li2023ultrare,zhang2024graph}. 

Approximate unlearning removes the target data’s influence by directly updating model parameters. 
Certified graph unlearning~\cite{chien2022certified} offers one of the earliest approximate unlearning frameworks for SGC~\cite{wu2019simplifying}, establishing formal removal guarantees. 
GraphGuard~\cite{wu2024graphguard} provides a complete framework to address data misuse in GNNs, incorporating gradient ascent-based unlearning as a key component. 
Such gradient ascent formulations can be approximated via Taylor expansion, motivating a family of influence function-based methods~\cite{koh2017understanding,wu2023gif,wu2023certified}. 
GIF~\cite{wu2023gif} adapts influence functions to graph unlearning across feature, node, and edge removal cases, while CEU~\cite{wu2023certified} specializes in edge unlearning and derives theoretical bounds for its guarantees. 
Building on these ideas, IDEA~\cite{dong2024idea} presents a general and theoretically grounded framework that provides removal guarantees for node, edge, and feature unlearning. 
Recent progress has improved the scalability~\cite{pan2023unlearning,li2024tcgu,yi2025scalable,yang2025erase,zhang2025dynamic}, model performance~\cite{li2024towards,zhang2025node,ding2025adaptive}, and completeness~\cite{kolipaka2025cognac,wu2025verification} of approximate unlearning, expanding its practical applicability. 
A distinct contribution is GNNDelete~\cite{cheng2023gnndelete}, which achieves unlearning through an intermediate node embedding mask inserted between GNN layers. 
This design provides a unique alternative to both retraining-based and parameter-update unlearning methods.

{\bf Adversarial Attacks for GNNs.} 
Adversarial attacks aim to degrade the prediction performance of GNNs by manipulating specific inputs or training data, which can be generally categorized into three prototypes: evasion, poisoning, and backdoor attacks~\cite{sun2022adversarial,zhang2024trustworthy,dai2024comprehensive}. 
Evasion attacks directly perturb the input graph to reduce model performance and have been shown to be effective against both white-box~\cite{dai2018adversarial,wu2019adversarial,xu2019topology} and black-box GNNs~\cite{dai2018adversarial,ma2021graph,geisler2021robustness}. 
Poisoning attacks insert or modify nodes and edges in the training graph, causing the trained model to learn corrupted representations and perform poorly~\cite{zugner2018adversarial,bojchevski2019adversarial,zugner2019adversarial, sun2020adversarial}. 
Backdoor attacks are more stealthy: they first poison the training data to implant hidden triggers that do not affect normal predictions but, when specific trigger patterns appear, lead to severe performance degradation~\cite{zhang2021backdoor,xi2021graph,dai2023unnoticeable,zhang2025robustness}. 
\textit{Our proposed attack fundamentally differs from these existing lines of research, as it explores a new attack surface in the context of graph unlearning, where mandatory data removal operations can be exploited to reduce model performance}. A more detailed discussion is shown in Section~\ref{sec:formal_prob_dfn}.

{\bf Negative Side Effects of Machine Unlearning.} 
Machine unlearning is especially critical in privacy-sensitive domains, such as autonomous driving~\cite{zhang2024smartcooper,fang2024pacp,goel2024corrective} and healthcare~\cite{fu2021sagn,lei2025watchguardian,hardan2025forget}, where users' right to remove their personal data carries real-world consequences. 
Although machine unlearning effectively removes the influence of undesired data from trained ML models, its potential negative side effects have recently attracted growing attention. 
The most direct side effect is a slight but inevitable degradation in model performance, which has been consistently observed across vision~\cite{chen2023boundary,li2024single,zhou2025decoupled}, language~\cite{yao2024large,zhang2024negative,zhang2025catastrophic}, and graph domains~\cite{wu2023gif,wu2023certified,dong2024idea}. 
In addition, several studies have shown that unlearning may be incomplete for the unlearned data~\cite{hu2024learn,zhang2025unlearning,wu2025verification,sui2025recalling} or can amplify the privacy leakage of remaining data (also known as the ``onion effect'')~\cite{carlini2022privacy,zhao2024makes}, creating new privacy attack surfaces. 
Early findings also reveal that machine unlearning can be exploited to implant or trigger backdoors in ML models~\cite{liu2024backdoor,huang2024uba,ren2025keeping}, activate previously hidden poisoned samples~\cite{di2023hidden}, or cause significant performance degradation under malicious requests~\cite{zhao2023static,huang2025unlearn,chen2025fedmua}. 
In this work, we investigate the adverse side effects of unlearning in the graph domain and examine how these effects can be amplified to launch adversarial attacks. A more detailed comparison between our work and related attack settings can be found in Section~\ref{sec:formal_prob_dfn} and \AppendixName~\ref{sec:discussion}.

\section{Preliminaries} \label{sec:prelim}

In this section, we present the notations used in this paper and give preliminaries on node classification and graph unlearning. 

{\bf Notations.} 
Bold uppercase letters (e.g., $\Xb$) denote matrices, bold lowercase letters (e.g., $\xb$) denote column vectors, and regular letters (e.g., $a$) denote scalars. We use $\boldsymbol{1}_n$ to denote column vectors with $n$ elements valued 1. 
A graph is denoted by $\Gcal := (\Ab, \Xb)$, where $\Ab \in \{0,1\}^{n \times n}$ is the adjacency matrix and $\Xb := [\xb_1, \ldots, \xb_n]^\top \in \R^{n \times d}$ is the node feature matrix. 
We define the node set $\Vcal := \{v_1, \ldots, v_n\}$ and the edge set $\Ecal \subseteq \Vcal \times \Vcal$, which are represented by the adjacency matrix $\Ab$, where $A_{i,j} = 1$ if $(v_i, v_j) \in \Ecal$ and $A_{i,j} = 0$ otherwise. 
We denote by $\Db \in \R^{n \times n}$ the degree matrix, where each diagonal element is $D_{i,i} := \sum_{j=1}^n A_{i,j}$ and all off-diagonal elements are zero. 

{\bf Semi-supervised Node Classification.} 
In this paper, we study a semi-supervised node classification problem in a transductive setting, which is widely adopted in real-world graph learning benchmarks~\cite{kipf2017semi,wu2020comprehensive}. 
In this setting, the input graph $\Gcal$ contains a subset of labeled training nodes $\Vcal_L := \{u_1, \ldots, u_{|\Vcal_L|}\} \subseteq \Vcal$, where each node $u_i$ is assigned a label $y_i \in \Ycal$. 
The remaining unlabeled test nodes form the subset $\Vcal_U$, with $\Vcal_U \cap \Vcal_L = \emptyset$. 
We denote the GNN prediction for a target node $v \in \Vcal$ as $f_v(\Ab, \Xb; \thetab)$, where $\Ab$ is the adjacency matrix used for message passing, $\Xb$ represents the node features, and $\thetab$ denotes the model parameters.

{\bf Graph Unlearning}. 
Graph unlearning aims to efficiently remove the influence of undesirable data from a trained GNN without full retraining~\cite{chen2022graph,wu2023gif}. Let the original training graph be $\Gcal_{\mathrm{orig}} = (\Ab_{\mathrm{orig}}, \Xb_{\mathrm{orig}})$, containing both retained and undesirable data. 
The original model parameters $\thetab_{\mathrm{orig}}$ are learned by minimizing the supervised loss over the labeled nodes:
\begin{align*}
    \thetab_{\mathrm{orig}} 
    := \argmin_{\thetab} 
    \sum_{v \in \Vcal_L} 
    \Lcal(f_v(\Ab_{\mathrm{orig}}, \Xb_{\mathrm{orig}}; \thetab), y_v),
\end{align*}
where $\Vcal_L$ is the set of labeled nodes, and $\Lcal(\cdot, \cdot)$ is a classification loss function such as cross-entropy~\cite{zhang2018generalized}.
 
The undesirable portion of $\Gcal_{\mathrm{orig}}$ is denoted by a subgraph $\Delta\Gcal = (\Delta\Vcal, \Delta\Ecal)$, where $\Delta\Vcal \subseteq \Vcal_{\mathrm{orig}}$ and $\Delta\Ecal \subseteq \Ecal_{\mathrm{orig}}$. 
After removing this undesirable subgraph, the remaining (cleaned) graph is $\Gcal_{\mathrm{un}} := (\Ab_{\mathrm{un}}, \Xb_{\mathrm{un}})$, with $\Vcal_{\mathrm{un}}:=\mathcal{V}_{\mathrm{orig}} \setminus \Delta \mathcal{V}$ and $\Ecal_{\mathrm{un}} := \Ecal_{\mathrm{orig}} \setminus \Delta\Ecal$. The objective of graph unlearning is to efficiently compute updated parameters $\thetab_{\mathrm{un}}$ through an unlearning algorithm $\textsc{Unlearn}$ (e.g., gradient ascent~\cite{wu2024graphguard,zhang2025catastrophic} or influence function–based methods~\cite{wu2023gif,wu2023certified}), such that $\thetab_{\mathrm{un}}$ approximates the parameters $\thetab_{\mathrm{retrain}}$ obtained by retraining on the cleaned graph $\Gcal_{\mathrm{un}}$, but at a substantially lower computational cost. 
Formally, the unlearning process can be expressed as: 
\begin{align*}
    \thetab_{\mathrm{un}} := & ~ \textsc{Unlearn}(f, \thetab_{\mathrm{orig}}, \mathcal{G}_{\mathrm{orig}}, \Delta\mathcal{G}) \notag \\
    \approx & ~ \underbrace{\mathop{\arg\min}_{\thetab} \sum_{v \in \mathcal{V}_L / \Delta\mathcal{V}} \Lcal(f_v(\Ab_{\mathrm{un}}, \Xb_{\mathrm{un}}; \thetab), y_v)}_{:=\thetab_{\mathrm{retrain}}},
\end{align*}
where the minimization process corresponds to retraining from scratch on the graph after removing $\Delta\Gcal$.

In this work, we focus on the {\bf node unlearning} setting~\cite{wu2023certified}, where $\Delta\Vcal \neq \emptyset$ and $\Delta\Ecal = \{(v_i, v_j) : v_i \in \Delta\Vcal \text{ or } v_j \in \Delta\Vcal\}$. 
This setting captures realistic scenarios such as user account deletion requests in social networks or recommender systems.

\section{Problem Formulation} \label{sec:problem}

In this section, we first present the threat model of the proposed GNN unlearning corruption attack and then formalize the corresponding problem definition.

\subsection{Threat Model} \label{sec:threat_model}

{\bf Attacker's Capability.} 
The attacker has access to the original training graph $\Gcal := (\Ab, \Xb)$ and is allowed to inject a small set of new nodes $\Delta\Vcal$ into it, resulting in an injected graph $\wt{\Gcal} := (\wt{\Ab}, \wt{\Xb})$. 
During the injection phase, the attacker can modify the features, edges, and labels of the injected nodes within a constrained budget. 
After the model owner trains a GNN model $f(\cdot, \cdot; \thetab_{\mathrm{inj}})$ on $\wt{\Gcal}$, the attacker can request the unlearning of the injected nodes $\Delta\Vcal$, producing an unlearned model $f(\cdot, \cdot; \thetab_{\mathrm{un}})$. 

We consider a {\bf black-box} and {\bf target-model-free} setting, where the attacker has no direct access to the target model architecture $f$, the parameters after node injection or unlearned parameters $\thetab_{\mathrm{inj}}$ and $\thetab_{\mathrm{un}}$, or the model outputs. 
The attacker is only aware that the model is a GNN and that a graph unlearning process is applied. 
To simulate this behavior, the attacker may train a surrogate model $\ov{f}$ and perform unlearning on it. 
Although $\ov{f}$ may differ from the target model in architecture, training, or unlearning details, it is assumed to share similar learning behavior with $f$.

{\bf Attacker's Goal.}
The attacker aims to induce unlearning-time performance degradation while keeping the attack stealthy.
Concretely, after the model owner trains on $\widetilde{\Gcal}$ and then performs unlearning, the attacker seeks:
\begin{itemize}[leftmargin=0.05\linewidth]
    \item {\bf (G1) Post-unlearning damage.} After unlearning $\Delta\Vcal$, the final unlearned model $f(\cdot,\cdot;\thetab_{\mathrm{un}})$ suffers a large performance drop on the unlabeled nodes (i.e., $\Vcal_U$).
    \item {\bf (G2) Pre-unlearning utility.} Before any unlearning, the injected model $f(\cdot,\cdot;\thetab_{\mathrm{inj}})$ maintains near-normal predictive performance (e.g., training/validation accuracy), so the negative effect of node injection is not easily detected.
    \item {\bf (G3) Stealthiness under benign unlearning.} Benign-only unlearning requests (i.e., requests that do not involve $\Delta\Vcal$) should not cause a significant additional performance drop beyond the natural effect of benign unlearning. 
\end{itemize}

{\bf Discussion.} 
Our threat-model assumptions are practical and mild. Concretely, we have: 
(i) \textbf{Access to $\Gcal$.} Many applications expose large portions of the graph (e.g., social or DeFi networks) via public APIs or crawlers, making acquisition of $\Gcal$ realistic~\cite{aliapoulios2021large,nielsen2022mumin,weber2019anti,chai2023towards}.  
(ii) \textbf{Feasibility of injection.} Injecting new nodes $\Delta\Vcal$ and manipulating their features and edges is feasible in many real-world systems. For instance, new accounts in social networks correspond to new nodes, editing profiles influence node features, and follows or friend requests create edges.  
(iii) \textbf{Unlearning injected nodes.} Data-removal or account-deletion mechanisms impose legally mandatory requests on model owners to remove user data and trigger unlearning, which is mandated by ``the right to be forgotten''~\cite{gdpr,ccpa,pipeda}. 
(iv) \textbf{Gray-box surrogate modeling.} Knowing only that a GNN and unlearning are used is a weak assumption, as the attacker has no access to the target model’s internals or outputs and can only train a surrogate $\ov{f}$ to approximate its learning and unlearning dynamics. Transferability of such attacks among unlearning methods is examined in Section~\ref{sec:comparison}. Together, these mild assumptions make the inject-then-unlearn attack both feasible and relevant in real-world graph-learning deployments.

\subsection{Graph Unlearning Corruption Attack} \label{sec:formal_prob_dfn}

Based on the threat model in Section~\ref{sec:threat_model}, we now formalize the graph unlearning corruption attack. We first define how an adversary injects nodes into the original graph, then state the attack objective as an optimization problem that chooses injected features and edges to maximize post-unlearning damage to the model.

\begin{definition}[Injected Graph]\label{dfn:inj_graph}
Let $\Gcal:= (\Ab, \Xb)$ be the original graph before node injection, with adjacency matrix $\Ab \in \{0,1\}^{n \times n}$ and node features $\Xb \in \R^{n\times d}$. The adversary injects a set of $m$ nodes $\Vcal$ (i.e., $|\Delta \Vcal| = m$), producing the injected graph $\wt{\Gcal} := (\wt{\Ab}, \wt{\Xb})$, such that 
\begin{align*}
\wt{\Xb} := & ~ 
\begin{bmatrix}
    \Xb \\ \Xb_{\mathrm{inj}}
\end{bmatrix}
\in\mathbb{R}^{(n+m) \times d}, \\ 
\wt{\Ab} := & ~
\begin{bmatrix}
    \Ab & \Ab_{\mathrm{inter}}^\top \\
    \Ab_{\mathrm{inter}} & \Ab_{\mathrm{intra}}
\end{bmatrix}
\in\{0,1\}^{(n+m)\times(n+m)}, 
\end{align*}
where $\Xb_{\mathrm{inj}}\in\mathbb{R}^{m \times d}$ denotes the injected node features, 
$\Ab_{\mathrm{intra}} \in \{0,1\}^{m\times m}$ encodes edges among the injected nodes, and $\Ab_{\mathrm{inter}} \in \{0,1\}^{m\times n}$ encodes edges between injected and original nodes. 
\end{definition}

In this paper, we assume that all injected nodes $\Delta\Vcal$ are labeled and are included in the training of the GNN on the injected graph $\wt{\Gcal}$, i.e., if $\wt{\Vcal}_L$ denotes the set of labeled nodes on $\wt{\Gcal}$, then $\Delta\Vcal \subseteq \wt{\Vcal}_L$. This assumption reflects practical scenarios in user-centric platforms (e.g., social networks, transaction networks) where graph data is derived from user-provided content. An adversary creating fake accounts (i.e., injected nodes) has full control over their user profiles, enabling them to assign arbitrary labels (e.g., self-declared interests or account types) and ensuring these nodes are integrated into the training set.

We then formalize the attack as an optimization problem, where the adversary crafts injected $\Xb_{\mathrm{inj}}$, $\Ab_{\mathrm{inter}}$, and $\Ab_{\mathrm{intra}}$ such that, after unlearning $\Delta\Vcal$, the model’s predictions on the remaining unlabeled nodes $\Vcal_U$ are maximally degraded.

%\newpage 

\begin{definition}[Graph Unlearning Corruption Attack]\label{dfn:goal_of_attack}
Let the original graph $\Gcal$ and the injected graph $\wt{\Gcal}$ be defined as in Definition~\ref{dfn:inj_graph}.
Let $f(\cdot,\cdot;\thetab_{\mathrm{inj}})$ denote a GNN trained on $\wt{\Gcal}$, and let integer $B \ge 0$ be the edge budget.
Let $\thetab_{\mathrm{un}}^{\Delta\Vcal} := \textsc{Unlearn}(f, \thetab_{\mathrm{inj}}, \wt{\Gcal}, \Delta\Vcal)$ denote the model parameters after unlearning $\Delta\Vcal$.
Let $\thetab_{\mathrm{un}}^{\Vcal_0} := \textsc{Unlearn}(f, \thetab_{\mathrm{inj}}, \wt{\Gcal}, \Vcal_0)$ denote the model parameters after unlearning a subset of training nodes $\Vcal_0 \subseteq \Vcal$ such that $\Vcal_0 \cap \Delta\Vcal = \emptyset$. 
The adversary’s objective is to solve the following maximization problem:  
\begin{align}
\max_{\Xb_{\mathrm{inj}}, \Ab_{\mathrm{inter}}, \Ab_{\mathrm{intra}}}\quad
& \sum_{v \in \Vcal_U}\Lcal(f_v(\wt{\Ab}, \wt{\Xb}; \thetab^{\Delta\Vcal}_{\mathrm{un}}), y_v) \label{eq:goal_obj}\\
\mathrm{s.t.}\quad
& \Ab_{\mathrm{inter}} \cdot \boldsymbol{1}_n + \Ab_{\mathrm{intra}} \cdot \boldsymbol{1}_m \le B \label{eq:goal_edge_budget}\\
& \sum_{v \in \Vcal_U}\Lcal(f_v(\wt{\Ab}, \wt{\Xb}; \thetab^{\Vcal_0}_{\mathrm{un}}), y_v)
\le \delta
\label{eq:goal_benign}
\end{align}
where $\Lcal(\cdot, \cdot)$ is an arbitrary classification loss function.

\end{definition}

In Definition~\ref{dfn:goal_of_attack}, the adversary optimizes the injected features $\Xb_{\mathrm{inj}}$ and edges $(\Ab_{\mathrm{inter}},\Ab_{\mathrm{intra}})$ to satisfy the three attacker's goals (G1--G3) in Section~\ref{sec:threat_model}.
The objective in Eq.~\eqref{eq:goal_obj} maximizes the loss on the unlabeled nodes $\Vcal_U$ after unlearning the injected nodes $\Delta\Vcal$, capturing post-unlearning damage (G1).
The budget constraint in Eq.~\eqref{eq:goal_edge_budget} limits the number of injected edges by $B$, enforcing stealthiness during injection (G2), which reflects practical settings where newly created accounts typically form only a small number of connections, and creating unusually many edges may trigger bot- or fraud-detection mechanisms.
Finally, Eq.~\eqref{eq:goal_benign} upper bounds the loss on $\Vcal_U$ after unlearning a benign node subset $\Vcal_0$ that excludes $\Delta\Vcal$, ensuring that benign-only unlearning requests do not cause a significant performance drop (G3).

Our unlearning corruption attack is qualitatively different from prevalent graph adversarial attack paradigms. We highlight the key distinctions below. Additional discussions on how our attack setting differs from prior unlearning-induced adversarial attacks are provided in \AppendixName~\ref{sec:discussion}.

\begin{remark}[Difference to poisoning attacks]
Poisoning attacks seek direct model performance degradation of the loss in Eq.~\eqref{eq:goal_obj} with respect to $\thetab_{\mathrm{inj}}$, without involving the unlearning step in Definition~\ref{dfn:goal_of_attack}. Equivalently, when the unlearning operator is the identity, i.e. $\textsc{Unlearn}(f,\thetab_{\mathrm{inj}},\widetilde{\Gcal},\Delta\Vcal)=\thetab_{\mathrm{inj}}$, 
Definition~\ref{dfn:goal_of_attack} reduces to the standard poisoning formulation. By contrast, our objective explicitly depends on the post-unlearning parameters $\thetab_{\mathrm{un}}$ returned by the $\textsc{Unlearn}$ procedure. Hence, injected nodes can be benign under $\thetab_{\mathrm{inj}}$ but induce large degradation only after unlearning, which makes our attack substantially more stealthy.
\end{remark}

\begin{remark}[Difference to Backdoor Attacks]
    Backdoor attacks implant a trigger so that the GNN model behaves abnormally when the trigger appears at inference (e.g., a specific subgraph or feature pattern). Such attacks can often be mitigated by modifying model inputs. In contrast, our unlearning corruption attacks do not depend on any inference-time trigger. Instead, they exploit the unlearning process itself, which is legally mandated by the right to be forgotten and thus difficult for model owners to reject.
\end{remark}
\section{Proposed Method} \label{sec:proposed}

Directly solving Definition~\ref{dfn:goal_of_attack} is difficult in practice for two reasons. 
First, the unlearning operator $\textsc{Unlearn}(\cdot)$ is treated as a black box, so gradients with respect to the injected features and edges are unavailable.
Second, the attacker does not observe ground-truth labels on the unlabeled set $\Vcal_U$ and thus cannot directly optimize the loss in Eq.~\eqref{eq:goal_obj} and address the constraint in Eq.~\eqref{eq:goal_benign}. 

In this section, we address these two challenges consecutively in Section~\ref{sec:approx_unlearn} and Section~\ref{sec:surrogate_obj}. We then describe the optimization process of our attack in Section~\ref{sec:optim_atk} and discuss scalability concerns in Section~\ref{sec:scale}.

\subsection{Approximating Black-box Unlearning} \label{sec:approx_unlearn}
%{\bf Approximating Black-box Unlearning.} 
To address the first challenge on the gradient computation of the unlearning operator $\textsc{Unlearn}(\cdot)$, we adopt a standard gradient-based formulation of the unlearning process~\cite{wu2023gif,wu2023certified,wu2024graphguard}, which updates model parameters to retain performance on clean labeled nodes $\Vcal_L\setminus\Delta\Vcal$ while forgetting the injected nodes $\Delta\Vcal$.
Specifically, we define the following gradient ascent unlearning objective:
\begin{align*}
    & ~ g_\gamma(\wt{\Ab}, \wt{\Xb}, \thetab, \Vcal_{\mathrm{un}}) \\
    := & ~ \sum_{v \in \Vcal_L \setminus \Vcal_{\mathrm{un}}} \Lcal(f_v(\wt{\Ab},\wt{\Xb};\thetab),y_v) 
    - \gamma \sum_{v \in \Vcal_{\mathrm{un}}}\Lcal(f_v(\wt{\Ab},\wt{\Xb};\thetab),y_v),
\end{align*}
where $\gamma>0$ balances retention and forgetting, and $\Vcal_{\mathrm{un}}$ denotes the node subset to be unlearned, which can be $\Delta \Vcal$ (for attack) or $\Vcal_0$ (for stealthiness) in Definition~\ref{dfn:goal_of_attack}. 

Therefore, the actual unlearned parameter $\thetab_{\mathrm{un}}^{\Vcal_{\mathrm{un}}}$ is the solution obtained via an arbitrary optimization solver $\textsc{Solve}$ initialized at the injected parameters $\thetab_{\mathrm{inj}}$:
\begin{align} \label{eq:theta_un_via_solve}
    \thetab_{\mathrm{un}}^{\Vcal_{\mathrm{un}}} = \textsc{Solve}(g_\gamma(\wt{\Ab}, \wt{\Xb}, \cdot, \Vcal_{\mathrm{un}}), \thetab_{\mathrm{inj}}). 
\end{align}

In this work, we employ a one-step approximation to instantiate the solver $\textsc{Solve}$, which estimates the unlearned parameters by performing a single gradient descent update from the initialization:
\begin{align}\label{eq:solve_via_one_step}
    \textsc{Solve}(g_\gamma, \thetab_{\mathrm{inj}}) \approx \thetab_{\mathrm{inj}} - \eta_{\mathrm{un}} \nabla_{\thetab} g_\gamma(\wt{\Ab}, \wt{\Xb}, \thetab_{\mathrm{inj}}, \Vcal_{\mathrm{un}}).
\end{align}

This approximation has two advantages. First, it requires only a single gradient step, making it efficient to evaluate,  and we empirically verify in Section~\ref{sec:comparison} that it provides a useful surrogate in practice. Second, since the formulation of $\textsc{Solve}$ consists entirely of differentiable operations, we can compute gradients through the unlearning process via standard automatic differentiation (i.e., gradient-of-gradient) in frameworks such as PyTorch, enabling end-to-end optimization of injected features and edges.

\subsection{Surrogate Objective via Pseudo Labels} \label{sec:surrogate_obj}
%{\bf Surrogate Objective via Pseudo Labels.} 
Since the attacker lacks access to the ground-truth labels of the unlabeled set $\Vcal_U$, we cannot directly optimize the threat model's objective (i.e., Eq.~\eqref{eq:goal_obj} and Eq.~\eqref{eq:goal_benign} in Definition~\ref{dfn:goal_of_attack}). 

To address this problem, we adopt a surrogate-modeling approach, where the attacker trains a surrogate GNN model $\ov{f}$ on the injected graph $\wt{\Gcal}$ and uses its predictions to construct pseudo labels $\ov{y}_v$ for $v\in\Vcal_U$. These pseudo-labels serve as a proxy for the victim model's behavior. 

Using the surrogate model $\ov{f}$, we formulate the final surrogate attack objective by combining the malicious goal in Eq.~\eqref{eq:goal_obj} and the stealthiness goal in Eq.~\eqref{eq:goal_benign} as follows: 
\begin{align}
& ~ \Lcal_{\mathrm{atk}}(\wt{\Ab}, \wt{\Xb}, \thetab_{\mathrm{inj}}) 
\notag \\
:=  & ~{\sum_{v\in\Vcal_U} 
\Lcal(\ov{f}_v(\wt{\Ab},\wt{\Xb};\thetab_{\mathrm{un}}^{\Delta\Vcal}), \ov{y}_v)} - \lambda {\sum_{v\in\Vcal_U} 
\Lcal(\ov{f}_v(\wt{\Ab},\wt{\Xb};\thetab^{\Vcal_0}_{\mathrm{un}}), \ov{y}_v)}, 
\label{eq:atk_loss}
\end{align}
where $\lambda \ge 0$ controls the strength of the stealthiness regularization term, $\thetab_{\mathrm{un}}^{\Delta\Vcal}$ represents the parameters after unlearning the injected nodes $\Delta\Vcal$ as in Definition~\ref{dfn:goal_of_attack}, and $\thetab_{\mathrm{un}}^{\Vcal_0}$ represents the parameters after unlearning the benign subset $\Vcal_0$ as defined in Definition~\ref{dfn:goal_of_attack}. Note that both sets of unlearned parameters are approximated via the one-step solver defined in Eq.~\eqref{eq:theta_un_via_solve} and Eq.~\eqref{eq:solve_via_one_step}. 

Considering the final surrogate goal in Eq.~\eqref{eq:atk_loss}, the first term promotes post-unlearning damage after unlearning $\Delta\Vcal$, while the second term discourages significant degradation under benign-only unlearning requests.

\subsection{Optimization} \label{sec:optim_atk}
%{\bf Optimization.} 
With the surrogate objective in Eq.~\eqref{eq:atk_loss}, we obtain a practical optimization problem over the injected features and edges. Specifically, we obtain the following optimization problem:

\begin{definition}[Relaxed problem for graph unlearning corruption]\label{dfn:bilevel_optim}
Let $\Xb_{\mathrm{inj}}\in\mathbb{R}^{m \times d}$, $\Ab_{\mathrm{intra}} \in \{0,1\}^{m\times m}$, and $\Ab_{\mathrm{inter}} \in \{0,1\}^{m\times n}$ be in Definition~\ref{dfn:inj_graph}. Under
the same setting as Definition~\ref{dfn:goal_of_attack} and with $\Lcal_{\mathrm{atk}}$ defined as in Eq.~\eqref{eq:atk_loss}, the adversary solves the following maximization problem: 
\begin{align*}
\max_{\Xb_{\mathrm{inj}}, \Ab_{\mathrm{inter}}, \Ab_{\mathrm{intra}}}\quad
& \Lcal_{\mathrm{atk}}(\wt{\Ab}, \wt{\Xb}, \thetab_{\mathrm{inj}}) \\
\mathrm{s.t.}\quad
& \Ab_{\mathrm{inter}} \cdot \boldsymbol{1}_n + \Ab_{\mathrm{intra}} \cdot \boldsymbol{1}_m \le B.
\end{align*}
\end{definition}

However, we cannot simply perform the optimization, since the injected parts for the adjacency matrix $\Ab_{\mathrm{inter}}$ and $\Ab_{\mathrm{intra}}$ are discrete, which is hard to optimize with simple gradient methods. To this end, we re-parameterize the edges using unconstrained free variables $\ov{\Ab}_{\mathrm{inter}} \in \mathbb{R}^{m \times n}$ and $\ov{\Ab}_{\mathrm{intra}} \in \mathbb{R}^{m \times m}$. We map these variables to the continuous range $[0, 1]$ via the sigmoid function, i.e., $\Ab = \mathrm{sigmoid}(\ov{\Ab})$. 
After optimizing these continuous variables, we project the final solution onto the discrete feasible set: 
%\begin{align*}
$
\mathcal{C} := \{ (\Ab_{\mathrm{inter}}, \Ab_{\mathrm{intra}}): \Ab_{\mathrm{intra}} \in \{0,1\}^{m\times m}, \Ab_{\mathrm{inter}} \in \{0,1\}^{m\times n},  \Ab_{\mathrm{inter}} \cdot \boldsymbol{1}_n + \Ab_{\mathrm{intra}} \cdot \boldsymbol{1}_m \le B \}
$. 
%.\end{align*}
We first derive the optimal projection onto the constraint set $\Ccal$ in Appendix~\ref{sec:proof}, and then present the detailed optimization procedure in Algorithm~\ref{alg:pgd}.

Notably, our method differs from standard Projected Gradient Descent (PGD), which enforces discrete constraints at every step. Applying discrete projection iteratively typically leads to a vanishing gradient problem, where small accumulated updates are zeroed out by rounding. For example, if an edge $\Ab_{i,j}$ is initialized at $0$, a gradient update might increase it to $0.2$, but the subsequent discrete projection would round it back to $0$, effectively erasing the update. Instead, our sigmoid-based reparameterization allows gradients to accumulate in the unconstrained space $\mathbb{R}$ throughout the optimization, and we apply a hard projection only at the final stage. 

\begin{algorithm}[!t]
\caption{Optimization Process for Graph Unlearning Corruption Attack}
\label{alg:pgd}
\begin{algorithmic}[1]
\Require Original graph $(\Ab,\Xb)$, injected parameters $\thetab_{\mathrm{inj}}$, budget $B$, learning rates $\eta_x, \eta_a$, steps $T$
\Ensure Injected node features $\Xb_{\mathrm{inj}}$, discrete injected edges $\Ab_{\mathrm{inter}},\Ab_{\mathrm{intra}}$
\State Initialize $\Xb_{\mathrm{inj}}\in\mathbb{R}^{m\times d}$ randomly
\State Initialize  $\ov{\Ab}_{\mathrm{inter}}\in\mathbb{R}^{m\times n}, \ov{\Ab}_{\mathrm{intra}}\in\mathbb{R}^{m\times m}$ randomly
\For{$t=1,\dots,T$}
    \State \emph{// Map to continuous range [0, 1]}
    \State $\Ab_{\mathrm{inter}} \leftarrow \mathrm{sigmoid}(\ov{\Ab}_{\mathrm{inter}})$; \quad $\Ab_{\mathrm{intra}} \leftarrow \mathrm{sigmoid}(\ov{\Ab}_{\mathrm{intra}})$
    \State Build $\wt{\Ab}=\begin{bmatrix}\Ab & \Ab_{\mathrm{inter}}^\top \\ \Ab_{\mathrm{inter}} & \Ab_{\mathrm{intra}}\end{bmatrix}$ and $\wt{\Xb}=\begin{bmatrix}\Xb \\ \Xb_{\mathrm{inj}}\end{bmatrix}$
    \State Evaluate attack loss $\Lcal_{\mathrm{atk}}$ (Eq.~\ref{eq:atk_loss}) via unrolled approximation
    \State Update $\Xb_{\mathrm{inj}}\leftarrow \Xb_{\mathrm{inj}}+\eta_x\nabla_{\Xb_{\mathrm{inj}}} \Lcal_{\mathrm{atk}}$
    \State Update $\ov{\Ab}_{\mathrm{inter}}\leftarrow \ov{\Ab}_{\mathrm{inter}}+\eta_a \nabla_{\ov{\Ab}_{\mathrm{inter}}} \Lcal_{\mathrm{atk}}$
    \State Update $\ov{\Ab}_{\mathrm{intra}}\leftarrow \ov{\Ab}_{\mathrm{intra}}+\eta_a \nabla_{\ov{\Ab}_{\mathrm{intra}}} \Lcal_{\mathrm{atk}}$
\EndFor
    \State \emph{// Projection to satisfy the edge budget}
    \State Compute final continuous $\Ab_{\mathrm{inter}} = \mathrm{sigmoid}(\ov{\Ab}_{\mathrm{inter}})$, $\Ab_{\mathrm{intra}} = \mathrm{sigmoid}(\ov{\Ab}_{\mathrm{intra}})$
    \For{each row $i = 1 ... m$ of concatenated $[\Ab_{\mathrm{inter}}, \Ab_{\mathrm{intra}}]$}
        \State Set top-$B$ largest entries to $1$ and set others to $0$
    \EndFor
\State \textbf{return} $\Xb_{\mathrm{inj}},\Ab_{\mathrm{inter}},\Ab_{\mathrm{intra}}$
\end{algorithmic}
\end{algorithm}

\subsection{Scalability} \label{sec:scale}
%{\bf Scalability via Neighborhood Sparsification.} 

We now analyze the computational complexity of the proposed optimization framework and introduce a sparsification strategy to ensure scalability on large graph datasets. Due to space limitations, we defer the detailed time complexity analysis of both the original algorithm and our efficient approximation to Appendix~\ref{sec:time}.

To address potential scalability bottlenecks, 
we exploit the locality of GNN aggregation mechanisms. The influence of the injected nodes $\Delta\Vcal$ on the target victims $\Vcal_U$ is primarily mediated through their local $k$-hop structural neighbors. 
Therefore, optimizing connections to nodes topologically distant from $\Vcal_U$ may yield negligible gradients and minimal impact on the attack objective. 

Based on this intuition, we propose a scalable approximation by restricting the optimization search space. Specifically, we enforce sparsity on the injected adjacency matrix $\wt{\Ab}$ by deactivating entries corresponding to distant nodes.
We define the candidate connection set $\mathcal{S}_{\mathrm{cand}}$ as the union of the $k$-hop neighborhoods of the target nodes:
$
    \mathcal{S}_{\mathrm{cand}} = \bigcup_{v \in \Vcal_U} \mathcal{N}_k(v).
$
We then constrain the optimization of $\ov{\Ab}_{\mathrm{inter}}$ such that $\ov{\Ab}_{\mathrm{inter}}[i, j]$ is trainable only if node $v_j \in \mathcal{S}_{\mathrm{cand}}$, while fixing all other entries to $0$. Additionally, to further reduce complexity, we fix $\ov{\Ab}_{\mathrm{intra}} = \mathbf{0}$, assuming that connections between injected nodes contribute marginally to the attack compared to connections with the victim's neighborhood.

\section{Experiments}\label{sec:exp_setting}

\begin{table*}[!ht]
\caption{\textbf{Main Comparison Results.} Original accuracy, unlearned accuracy, and accuracy drop ($\Delta$Acc) across datasets. Best results are in \textbf{bold}, and second-best are \underline{underlined}. A complete version with standard deviation can be found in Table~\ref{tab:main_new_full} in Appendix~\ref{sec:more_results}.}
\label{tab:main_new}
\vskip -0.1in
\centering
\resizebox{0.98\linewidth}{!}{
\begin{tabular}{llccc|ccc|ccc|ccc}
\toprule[1.5pt]
\multirow{2}{*}{\textbf{\shortstack{Unlearn\\Method}}} & \multirow{2}{*}{\textbf{Attack}} 
  & \multicolumn{3}{c}{\textbf{Cora}} 
  & \multicolumn{3}{c}{\textbf{Citeseer}} 
  & \multicolumn{3}{c}{\textbf{Pubmed}} 
  & \multicolumn{3}{c}{\textbf{Flickr}} \\
\cmidrule(lr){3-5} \cmidrule(lr){6-8} \cmidrule(lr){9-11} \cmidrule(lr){12-14}
 & & Original$\uparrow$ & Unlearned$\downarrow$ & $\Delta$Acc$\downarrow$ 
   & Original$\uparrow$ & Unlearned$\downarrow$ & $\Delta$Acc$\downarrow$ 
   & Original$\uparrow$ & Unlearned$\downarrow$ & $\Delta$Acc$\downarrow$ 
   & Original$\uparrow$ & Unlearned$\downarrow$ & $\Delta$Acc$\downarrow$ \\
\midrule[1pt]
% ==================== GIF ====================
\multirow{10}{*}{\textbf{GIF}}
 & NoAttack & $0.8502$ & $0.8458$ & $0.0044$ & $0.7357$ & $0.7201$ & $0.0156$ & $0.8308$ & $0.8180$ & $0.0129$ & $0.4355$ & $0.4348$ & $0.0007$ \\
 \cmidrule(lr){2-14}
 & Random & $0.8317$ & $0.8435$ & $-0.0118$ & $0.7201$ & $0.7363$ & $-0.0162$ & $0.8267$ & $0.7781$ & $0.0486$ & $0.4355$ & $0.4345$ & $0.0010$ \\
 & Copy & $0.8494$ & $0.8450$ & $0.0044$ & $0.7375$ & $0.7411$ & $-0.0036$ & $0.8308$ & $0.8302$ & $0.0006$ & $0.3648$ & $0.3644$ & $0.0004$ \\
 & NewCopy & $0.8317$ & $0.8435$ & $-0.0118$ & $0.7375$ & $0.7376$ & $-0.0001$ & $0.8362$ & $0.8316$ & $0.0046$ & $0.4438$ & $0.4435$ & $0.0003$ \\
 & TestCopy & $0.8303$ & $0.8295$ & $0.0007$ & $0.7153$ & $0.7243$ & $-0.0090$ & $0.8440$ & $0.7991$ & $0.0449$ & $0.4323$ & $0.3961$ & $\underline{0.0362}$ \\
 & TestLink & $0.7801$ & $0.8332$ & $-0.0531$ & $0.7279$ & $0.7628$ & $-0.0348$ & $0.8165$ & $0.8329$ & $-0.0163$ & $0.4319$ & $0.4317$ & $0.0002$ \\
 \cmidrule(lr){2-14}
 & SBA & $0.8259$ & $0.8233$ & $0.0026$ & $0.7140$ & $0.7134$ & $0.0006$ & $0.8262$ & $0.8261$ & $0.0001$ & $0.4339$ & $0.4220$ & $0.0119$ \\
 & UGBA & $0.8311$ & $0.7091$ & $\underline{0.1220}$ & $0.7173$ & $0.7158$ & $0.0015$ & $0.8247$ & $0.8245$ & $0.0002$ & $0.4208$ & $0.4167$ & $0.0041$ \\
 & TDGIA & $0.7538$ & $0.7231$ & $0.0307$ & $0.6917$ & $0.6556$ & $\underline{0.0361}$ & $0.8181$ & $0.5791$ & $\underline{0.2390}$ & $0.4057$ & $0.4035$ & $0.0022$ \\
 \cmidrule(lr){2-14}
 & \textbf{OptimAttack} & $0.8162$ & $0.6103$ & $\textbf{0.2059}$ & $0.7381$ & $0.2066$ & $\textbf{0.5315}$ & $0.8351$ & $0.2443$ & $\textbf{0.5908}$ & $0.4320$ & $0.2860$ & $\textbf{0.1460}$ \\
\midrule[1pt]

% ==================== CEU ====================
\multirow{10}{*}{\textbf{CEU}}
 & NoAttack & $0.8214$ & $0.8192$ & $0.0022$ & $0.7243$ & $0.7075$ & $0.0168$ & $0.8084$ & $0.7956$ & $0.0128$ & $0.4440$ & $0.4441$ & $-0.0001$ \\
 \cmidrule(lr){2-14}
 & Random & $0.8103$ & $0.8244$ & $-0.0140$ & $0.7195$ & $0.7321$ & $-0.0126$ & $0.7951$ & $0.3849$ & $0.4102$ & $0.4318$ & $0.4316$ & $0.0002$ \\
 & Copy & $0.8221$ & $0.8220$ & $0.0001$ & $0.7195$ & $0.7159$ & $0.0036$ & $0.8049$ & $0.7467$ & $0.0582$ & $0.4432$ & $0.3987$ & $\underline{0.0444}$ \\
 & NewCopy & $0.8089$ & $0.8229$ & $-0.0140$ & $0.7195$ & $0.7207$ & $-0.0012$ & $0.8041$ & $0.7831$ & $0.0210$ & $0.4317$ & $0.4316$ & $0.0001$ \\
 & TestCopy & $0.8155$ & $0.8177$ & $-0.0022$ & $0.7207$ & $0.7205$ & $0.0002$ & $0.8121$ & $0.3869$ & $\underline{0.4252}$ & $0.4420$ & $0.4426$ & $-0.0006$ \\
 & TestLink & $0.7498$ & $0.8185$ & $-0.0686$ & $0.6967$ & $0.7411$ & $-0.0444$ & $0.7438$ & $0.3274$ & $0.4164$ & $0.4322$ & $0.4321$ & $0.0000$ \\
 \cmidrule(lr){2-14}
 & SBA & $0.8037$ & $0.7993$ & $0.0044$ & $0.6941$ & $0.6947$ & $-0.0006$ & $0.7985$ & $0.5682$ & $0.2303$ & $0.4232$ & $0.4230$ & $0.0002$ \\
 & UGBA & $0.8085$ & $0.6909$ & $\underline{0.1176}$ & $0.6926$ & $0.6926$ & $-0.0000$ & $0.7877$ & $0.6562$ & $0.1315$ & $0.4219$ & $0.4218$ & $0.0001$ \\
 & TDGIA & $0.5309$ & $0.5279$ & $0.0030$ & $0.5158$ & $0.4857$ & $\underline{0.0301}$ & $0.6116$ & $0.6229$ & $-0.0114$ & $0.3912$ & $0.3850$ & $0.0062$ \\
 \cmidrule(lr){2-14}
 & \textbf{OptimAttack} & $0.7904$ & $0.5801$ & $\textbf{0.2103}$ & $0.7267$ & $0.2276$ & $\textbf{0.4991}$ & $0.8102$ & $0.2146$ & $\textbf{0.5956}$ & $0.4322$ & $0.2858$ & $\textbf{0.1464}$ \\
\midrule[1pt]

% ==================== GA ====================
\multirow{10}{*}{\textbf{GA}}
 & NoAttack & $0.8443$ & $0.8177$ & $0.0266$ & $0.7417$ & $0.6805$ & $0.0612$ & $0.8320$ & $0.8158$ & $0.0162$ & $0.4461$ & $0.4336$ & $0.0125$ \\
 \cmidrule(lr){2-14}
 & Random & $0.8295$ & $0.5166$ & $\underline{0.3129}$ & $0.7171$ & $0.4613$ & $0.2559$ & $0.8279$ & $0.7041$ & $\underline{0.1238}$ & $0.4433$ & $0.4317$ & $0.0116$ \\
 & Copy & $0.8524$ & $0.8258$ & $0.0266$ & $0.7321$ & $0.7177$ & $0.0144$ & $0.8304$ & $0.8291$ & $0.0013$ & $0.2125$ & $0.4331$ & $-0.2206$ \\
 & NewCopy & $0.8325$ & $0.8635$ & $-0.0310$ & $0.7393$ & $0.6859$ & $0.0535$ & $0.8371$ & $0.8334$ & $0.0038$ & $0.3615$ & $0.4298$ & $-0.0683$ \\
 & TestCopy & $0.8273$ & $0.7845$ & $0.0428$ & $0.7237$ & $0.7117$ & $0.0120$ & $0.8455$ & $0.8092$ & $0.0363$ & $0.4412$ & $0.4014$ & $\underline{0.0398}$ \\
 & TestLink & $0.7786$ & $0.8391$ & $-0.0605$ & $0.7243$ & $0.7435$ & $-0.0192$ & $0.8176$ & $0.8338$ & $-0.0161$ & $0.4054$ & $0.4183$ & $-0.0128$ \\
 \cmidrule(lr){2-14}
 & SBA & $0.8288$ & $0.8348$ & $-0.0059$ & $0.7155$ & $0.6845$ & $0.0310$ & $0.8265$ & $0.8271$ & $-0.0006$ & $0.4339$ & $0.4220$ & $0.0119$ \\
 & UGBA & $0.8322$ & $0.7767$ & $0.0555$ & $0.7152$ & $0.4262$ & $\underline{0.2890}$ & $0.8252$ & $0.8285$ & $-0.0033$ & $0.4208$ & $0.4167$ & $0.0041$ \\
 & TDGIA & $0.7397$ & $0.7512$ & $-0.0115$ & $0.6911$ & $0.6830$ & $0.0081$ & $0.7924$ & $0.7905$ & $0.0019$ & $0.4055$ & $0.4002$ & $0.0053$ \\
 \cmidrule(lr){2-14}
 & \textbf{OptimAttack} & $0.8258$ & $0.4790$ & $\textbf{0.3469}$ & $0.7345$ & $0.4138$ & $\textbf{0.3207}$ & $0.8279$ & $0.6798$ & $\textbf{0.1481}$ & $0.4318$ & $0.3012$ & $\textbf{0.1306}$ \\
\bottomrule[1.5pt]
\end{tabular}
}
\vskip -0.1in
\end{table*}

In this section, we show our experimental setup and report the main empirical results of our study.

\subsection{Experiment Settings}
{\bf Datasets.} We evaluate our attack on four standard graph machine learning benchmarks: Cora~\cite{yang2016revisiting}, Citeseer~\cite{yang2016revisiting}, Pubmed~\cite{yang2016revisiting}, and Flickr~\cite{zeng2020graphsaint}. Following the setup of the well-established graph unlearning method GIF~\cite{wu2023gif}, 90\% of the nodes in each dataset are used for training, and the remaining 10\% for testing.

{\bf Victim Model.} We use a two-layer GCN~\cite{kipf2017semi} as the victim model for node classification. The GCN is implemented with the hyperparameters and training protocol from GIF~\cite{wu2023gif} to ensure consistency and reproducibility. To construct the poisoned training graph, we insert injected nodes equal to 5\% of the total node count, and allocate an edge budget of 5 for each injected node. After training the GCN on the poisoned graph, we apply GIF~\cite{wu2023gif}, CEU~\cite{wu2023certified}, and Gradient Ascent (GA)~\cite{wu2024graphguard} to unlearn the injected nodes, thereby triggering our unlearning-induced adversarial attack. All the unlearning methods are run with their official settings to ensure faithful reproduction.

{\bf Baselines.} As discussed in Section~\ref{sec:formal_prob_dfn}, our proposed attack operates under a fundamentally different setting from existing backdoor and poisoning attacks in general-purpose graph learning, making direct comparison infeasible. Similarly, current unlearning-induced adversarial attacks are not directly comparable due to the novelty of our setup (see \AppendixName~\ref{sec:discussion} for a detailed discussion). Therefore, we compare our optimization-based attack with several intuitive baselines, namely {\bf Random}, {\bf Copy}, {\bf NewCopy}, {\bf TestCopy}, {\bf TestLink}, and we defer their implementation details to Appendix~\ref{sec:impl_detail}. 

Although previous adversarial and backdoor node injection attacks are not specifically designed for our setting, we compare our method against representative GNN node injection attacks, including \textbf{SBA}~\cite{zhang2021backdoor}, \textbf{UGBA}~\cite{dai2023unnoticeable}, and \textbf{TDGIA}~\cite{zou2021tdgia}. We utilize their official implementations to inject labeled nodes into the training graph. To ensure a fair comparison, we align the train-test splits, injection budgets (nodes and edges), and feature normalization with our experimental setup.

For completeness, we also include a baseline that performs standard node unlearning without node injection, denoted as {\bf NoAttack}. This baseline verifies that any observed performance degradation is caused by the attack itself rather than the unlearning process.

{\bf Evaluation Metrics.} To comprehensively evaluate stealthiness and attack effectiveness, we report three metrics: (1) {\it Original Accuracy}, which measures model performance on the injected graph before unlearning to ensure that the attack remains stealthy and does not degrade pre-unlearning utility; (2) {\it Unlearned Accuracy}, which measures performance after the unlearning request is processed; and (3) {\it Accuracy Drop ($\Delta$Acc)}, our primary metric, which quantifies the attack's success by calculating the difference between original and unlearned accuracy (i.e., $\Delta\mathrm{Acc} = \mathrm{Original} - \mathrm{Unlearned}$). A higher $\Delta$Acc indicates a more effective corruption attack.

{\bf Implementation Details.} Due to space limitations, we defer the details of our proposed OptimAttack to Appendix~\ref{sec:impl_detail}. 

\subsection{Unlearning Corruption Attack Performance}\label{sec:comparison}

In this study, we present a comprehensive comparison of our proposed \textbf{OptimAttack} against the baselines described in Section~\ref{sec:exp_setting}. The main results are reported in Table~\ref{tab:main_new}. All experiments were repeated 5 times; for brevity, we report the mean values here, while the full results with standard deviations are provided in Appendix~\ref{sec:more_results}. Our key findings are summarized as follows:
\begin{itemize}[leftmargin=0.05\linewidth]\item \textbf{Significant Post-Unlearning Damage (G1).} Our proposed {OptimAttack} consistently achieves the highest accuracy drop ($\Delta$Acc) across all datasets and unlearning algorithms, successfully satisfying the primary adversarial goal (G1). For instance, on the Pubmed dataset, OptimAttack induces a massive performance collapse (e.g., $\sim$59\% drop under GIF and CEU), whereas heuristic baselines like {Random} or {Copy} often result in negligible or even negative drops. This superiority stems from our bi-level optimization framework: unlike intuitive baselines or standard poisoning attacks (e.g., SBA, UGBA) that are agnostic to the unlearning process, our method explicitly optimizes for post-unlearning degradation. This also validates the effectiveness of our one-step gradient approximation in capturing the dynamics of black-box unlearning operators.

\item \textbf{Stealthiness via Pre-Unlearning Utility (G2).} Regarding the second goal (G2), we observe that the \textit{Original Accuracy} of models poisoned by OptimAttack remains comparable to those under {NoAttack} or other baselines. For example, on Citeseer under the GIF setting, the original accuracy for OptimAttack ($0.7381$) is very close to the clean baseline ($0.7357$). This confirms that our attack is highly stealthy: the injected nodes behave benignly during the initial training phase and do not trigger a performance collapse until the unlearning request is executed.

\begin{table*}[!ht]
\centering
\caption{\textbf{Ablation Study on Stealthiness Regularization.} We compare our full method (``w/ Reg.'') against a variant of Eq.~\eqref{eq:atk_loss} without the stealthiness loss term ($\lambda=0$). \textbf{Benign F1} measures model utility after unlearning a random benign node; a higher value indicates better stealthiness (G3). \textbf{Unlearned F1} measures the attack effectiveness (G1). }
\vskip -0.1in
\label{tab:ablation_stealthiness}
\resizebox{1\textwidth}{!}{
\begin{tabular}{l|ccc|ccc|ccc}
\toprule
\multirow{2}{*}{\textbf{Method}} & \multicolumn{3}{c|}{\textbf{Cora}} & \multicolumn{3}{c|}{\textbf{Citeseer}} & \multicolumn{3}{c}{\textbf{Pubmed}} \\
 & Original F1$\uparrow$ & Unlearned F1$\downarrow$ & Benign F1$\uparrow$ & Original F1$\uparrow$ & Unlearned F1$\downarrow$ & Benign F1$\uparrow$ & Original F1$\uparrow$ & Unlearned F1$\downarrow$ & Benign F1$\uparrow$ \\
\midrule
w/o Reg. ($\lambda=0$) & $0.8163 \pm 0.0045$ & $0.5877 \pm 0.0480$ & $0.7581 \pm 0.0035$ & $0.7380 \pm 0.0059$ & $0.2059 \pm 0.0023$ & $0.6517 \pm 0.0043$ & $0.8349 \pm 0.0007$ & $0.2513 \pm 0.1236$ & $0.7877 \pm 0.0012$ \\
w/ Reg. (Ours) & $0.8162 \pm 0.0043$ & $0.6103 \pm 0.0600$ & $\mathbf{0.8053 \pm 0.0027}$ & $0.7381 \pm 0.0064$ & $0.2066 \pm 0.0029$ & $\mathbf{0.7286 \pm 0.0043}$ & $0.8351 \pm 0.0009$ & $0.2443 \pm 0.1372$ & $\mathbf{0.8303 \pm 0.0013}$ \\
\bottomrule
\end{tabular}
}
\vskip -0.1in
\end{table*}

\begin{figure*}[!ht]
    \centering
    \vskip -0.15in
    \begin{subfigure}{0.28\linewidth}
        \centering
        \includegraphics[width=\linewidth]{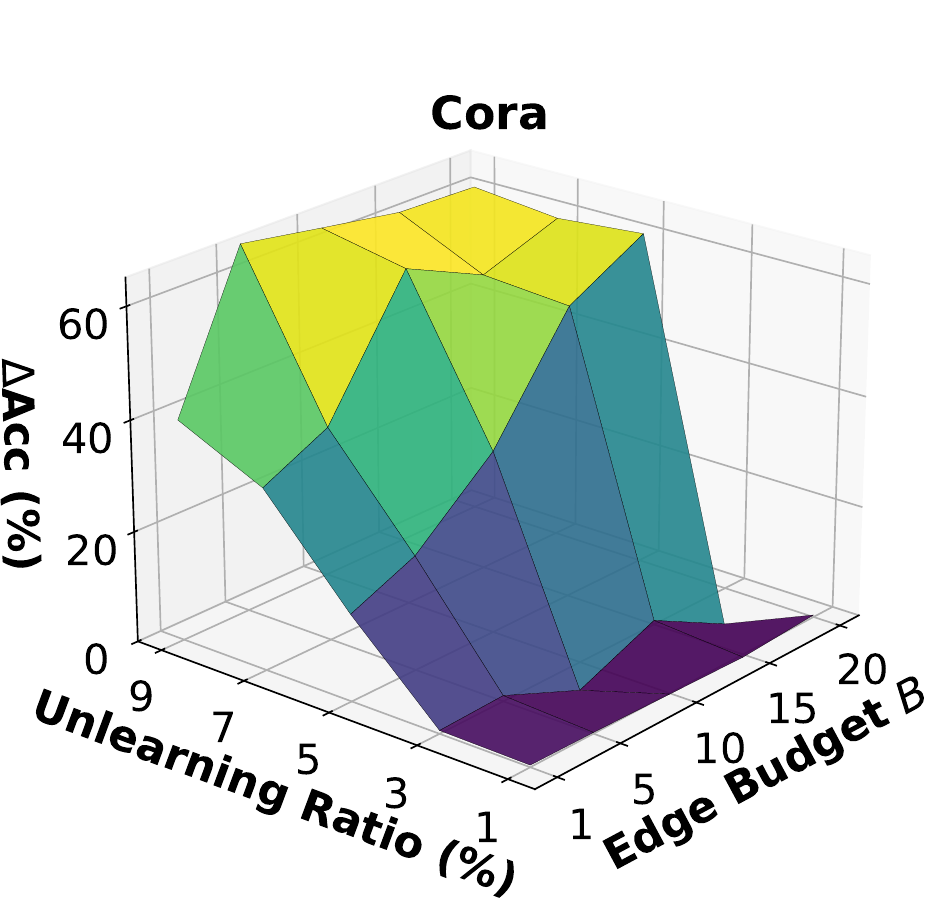}
        %\caption{Cora}
        \label{fig:budget_cora}
    \end{subfigure}
    \hfill
    \begin{subfigure}{0.28\linewidth}
        \centering
        \includegraphics[width=\linewidth]{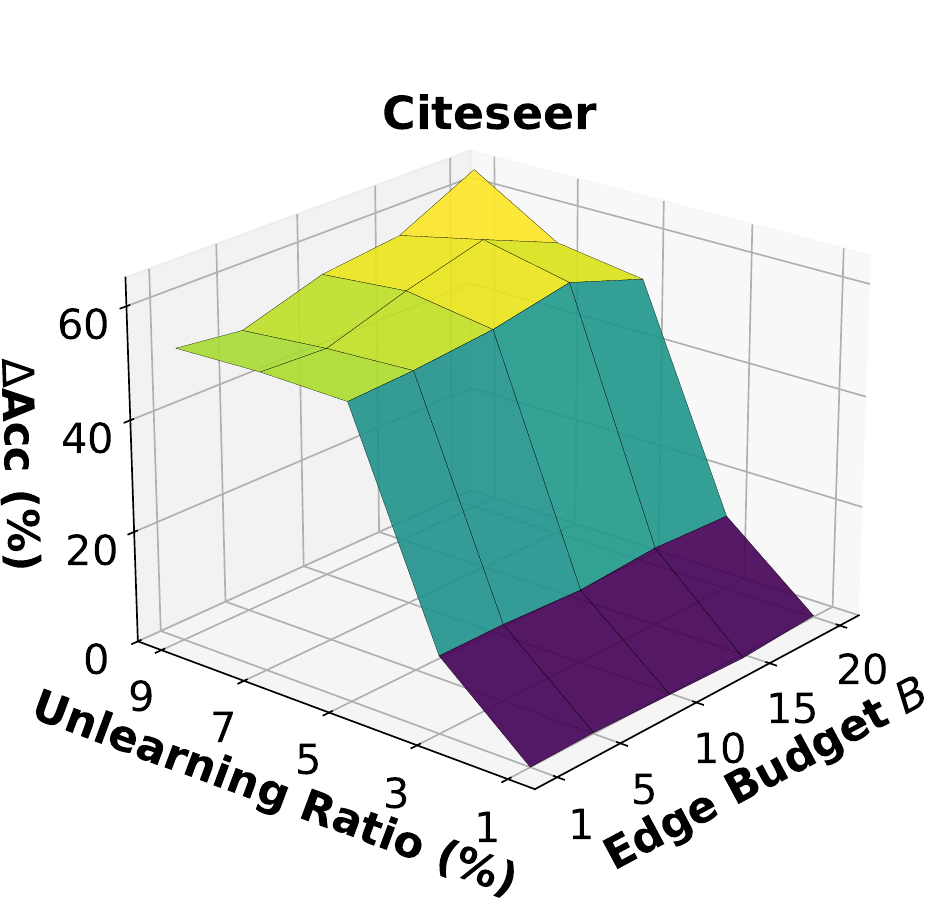}
        %\caption{Citeseer}
        \label{fig:budget_citeseer}
    \end{subfigure}
    \hfill
    \begin{subfigure}{0.28\linewidth}
        \centering
        \includegraphics[width=\linewidth]{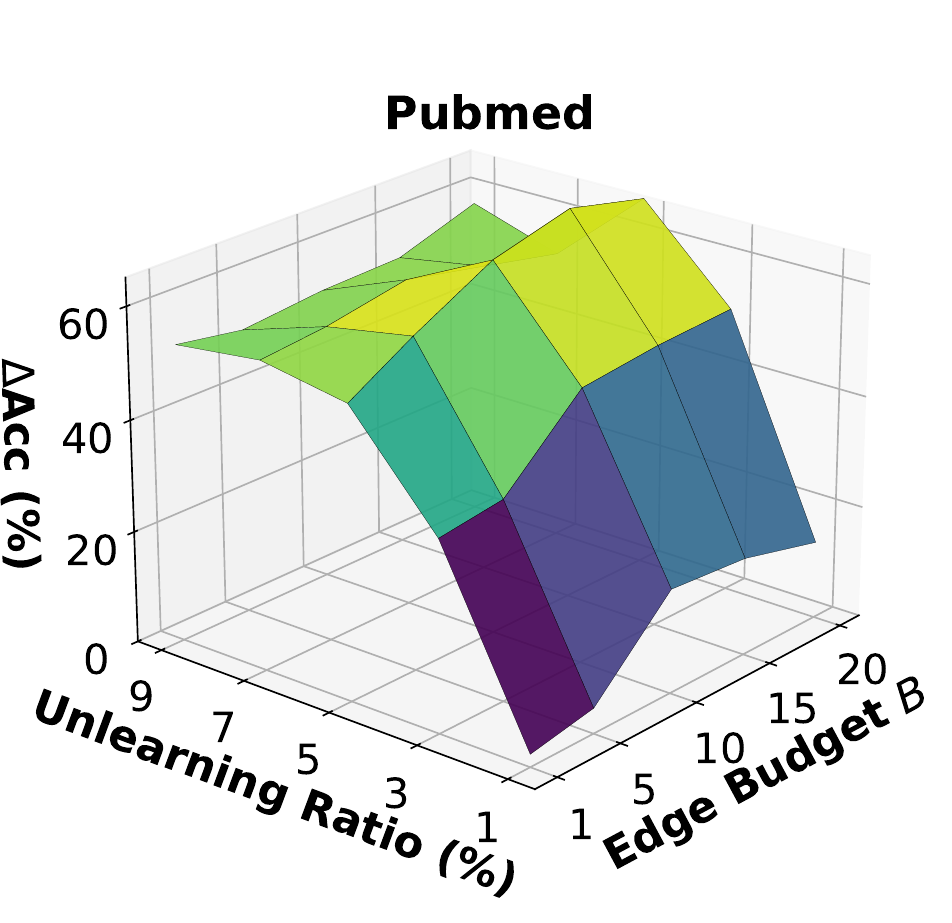}
        %\caption{Pubmed}
        \label{fig:budget_pubmed}
    \end{subfigure}
    \vskip -0.26in
    
\caption{\textbf{Impact of Attack Budget.} The attack performance ($\Delta$Acc) of OptimAttack across varying edge budgets ($B$) and unlearning ratios.}    
\vskip -0.15in
\label{fig:budget}
\end{figure*}

\item \textbf{Transferability Across Unlearning Mechanisms.} Our method demonstrates strong transferability across diverse unlearning algorithms. Although we utilize Gradient Ascent (GA) as the unlearning mechanism for the surrogate model during optimization, the generated attack vectors effectively degrade models unlearned via GIF (influence functions) and CEU (certified unlearning). This indicates that our surrogate modeling strategy effectively captures the general vulnerabilities of the unlearning problem, allowing the attack to succeed even when the victim's specific unlearning algorithm is unknown.
\item \textbf{Scalability via Neighborhood Restriction.} On the large-scale Flickr dataset, we employed a scalable approximation strategy that optimizes connections to only a random 10\% subset of test nodes and their immediate 1-hop neighbors. Despite this restricted search space, {OptimAttack} still achieves a substantial accuracy drop ($\sim$14.6\%), significantly outperforming baselines. This result suggests that our neighborhood restriction framework is highly efficient and capable of scaling to larger graphs. While these results demonstrate potential scalability, we emphasize that the primary contribution of this work is to establish the unlearning corruption attack surface itself, rather than optimizing for extreme scalability.
\end{itemize}

\subsection{Ablation Studies}

{\bf Impact of Unlearning Benign Nodes (G3).} To verify that our proposed loss function effectively addresses the goal of stealthiness under benign unlearning (G3), we conducted an ablation study by removing the second term in Eq.~(6) (i.e., setting $\lambda=0$). The results, presented in Table~\ref{tab:ablation_stealthiness}, demonstrate the critical role of this regularization. Without it (w/o Reg.''), the model suffers significant collateral damage when unlearning benign nodes: for instance, on Citeseer, the \textit{Benign F1} drops sharply to $0.6517$, compared to the original accuracy of $0.7380$. This indicates that without explicit constraints, the corruption attack destabilizes the model globally. In contrast, our full method (w/ Reg.'') maintains a \textit{Benign F1} of $0.7286$, almost recovering the original performance. Crucially, this improvement in stealthiness comes with negligible cost to attack effectiveness (G1), as the \textit{Unlearned F1} remains comparably low across all datasets. This confirms that the regularization term successfully isolates the corruption to the target unlearning request, satisfying G3.

\textbf{Impact of Attack Budget.} In Figure~\ref{fig:budget}, we evaluate how the performance degradation ($\Delta$Acc) of the proposed OptimAttack changes under varying attack budgets, specifically the node unlearning ratio and the edge injection budget $B$ per node. We evaluate unlearning ratios of $\{1\%, 3\%, 5\%, 7\%, 9\%\}$ and edge budgets of $\{1, 5, 10, 15, 20\}$, where $5\%$ and $B=5$ represent our default experimental setting. As shown in the figure, OptimAttack is highly stable and resource-efficient. While the damage ($\Delta$Acc) naturally increases as the attacker is granted more budget, the attack effectiveness quickly plateaus. Notably, at our strictly limited default setting (a $5\%$ unlearning ratio and just $5$ injected edges), the attack already achieves near-maximum performance degradation on all datasets.

{\bf Victim Model Transferrability.} We evaluate whether our unlearning corruption attack, optimized using a surrogate 2-layer GCN, can transfer to other GNN architectures. Due to space limitations, the results are deferred to Figure~\ref{fig:model_transferability} in Appendix~\ref{sec:more_results}. 
\section{Conclusion}

In this paper, we introduced the concept of {unlearning corruption attacks}, a novel adversarial setting where legally mandated data deletion requests are weaponized to degrade GNN performance. We formalized this threat through three key adversarial goals: (G1) maximizing post-unlearning damage, (G2) maintaining pre-unlearning utility, and (G3) ensuring stealthiness under benign unlearning requests. To achieve these goals, we proposed a bi-level optimization framework that overcomes the technical challenges of black-box unlearning and label scarcity via gradient-based approximations and surrogate modeling with pseudo-labels. While extreme scalability was not the primary objective of this study, we introduced an intuitive neighborhood restriction approach. In future work, we plan to refine this approximation to investigate scalable attack strategies for billion-scale graphs. %Additionally, we aim to explore the generalizability of these unlearning vulnerabilities to the emerging paradigm of Graph Foundation Models.

%\newpage
\ifdefined\isarxiv

\else

\clearpage
\bibliographystyle{ACM-Reference-Format}
\bibliography{ref}
\fi


\appendix
\newpage

\twocolumn[
\begin{center}
    {\LARGE \bf Appendix}
\end{center}
%\vskip 0.1in
]

\section*{List of Contents}
In this appendix, we provide the following additional information: 

\begin{itemize}[leftmargin=0.05\linewidth]
    \applist
\end{itemize}

\appsection{Discussion}{sec:discussion}

In this section, we discuss the differences between our attack settings and previous works. 

%\subsection{Comparison to Other Unlearning-induced Attacks}

{\bf Degradation Amplification Attacks.} A broad class of attacks~\cite{zhao2023static,chen2025fedmua,huang2025unlearn} aim to submit adversarial unlearning requests to machine learning services, intentionally causing the unlearned model to suffer significant performance degradation. A more stealthy variant of these attacks, known as unlearning-induced backdoors~\cite{liu2024backdoor,ren2025keeping}, manipulates the model such that its performance degrades only when a specific trigger appears in the input (e.g., certain pixels in images~\cite{liu2024backdoor} or sentences with abnormal words or typos in text~\cite{ren2025keeping}).  

However, these attacks have been developed in image~\cite{zhao2023static,chen2025fedmua,liu2024backdoor,huang2025unlearn} or language~\cite{ren2025keeping} domains, where a key assumption holds: {\it the unlearning request does not need to exactly match any original training sample.} For example, one may train a model on images of cats and dogs but later request unlearning for an image of a tiger, or train on the text of \textit{Harry Potter} and request unlearning for a new sentence not found in the book. This assumption does not naturally hold in graph learning, where unlearning must specify existing nodes or edges in the training graph~\cite{wu2023gif,cheng2023gnndelete,wu2023certified,wu2024graphguard}. For instance, in a social network, an adversary can only request the deletion of accounts they control, instead of arbitrary ones such as celebrity accounts. Consequently, these previous attacks are not directly applicable in the context of graph learning.

\begin{remark}[Difference to Degradation Amplification Attacks]
These attacks assume the flexibility to submit arbitrary unlearning requests (e.g., data points not present in the training set), whereas graph learning restricts unlearning to existing graph elements (i.e., nodes or edges). Our inject-then-unlearn setting thus represents a more realistic adversarial scenario for graphs, and these previous attacks are not directly comparable to ours.
\end{remark}

{\bf Camouflaged Poisoning Attacks.} Camouflaged poisoning attacks~\cite{di2023hidden,huang2024uba} follow an inject-then-unlearn paradigm similar to ours. They first poison the training dataset to induce performance degradation~\cite{di2023hidden} or implant a backdoor~\cite{huang2024uba}, and then inject camouflage samples to conceal the malicious effects. Once these camouflage samples are unlearned, the hidden effects can be reactivated. However, these methods are designed for Euclidean, non-graph data, where injected samples are independent and not interconnected, and their attack settings differ from ours, making them not directly comparable. Specifically, HiddenPoison~\cite{di2023hidden} assumes a white-box setting, whereas our approach is target-model free and relies on transferable surrogate modeling, and UBA-Inf~\cite{huang2024uba} does not exploit the graph structure. In contrast, our attack focuses on a more challenging and different black-box setting in the graph domain, jointly optimizing both the injected nodes and connectivity.

\begin{remark}[Difference to Camouflaged Poisoning Attacks]
Existing camouflaged poisoning attacks operate in Euclidean domains with white-box access, while our black-box graph setting jointly optimizes injected nodes and their connectivity, making the problem more realistic and challenging.
\end{remark}

\appsection{Optimal Projection onto the Constraint Set}{sec:proof}

To compute the projection onto the set $\Ccal$, which is a central problem mentioned in Section~\ref{sec:optim_atk}, we have the following proposition:

\begin{proposition}[Optimality of Greedy Projection]
Let $\ov{\Ab}_{\mathrm{intra}} \in [0,1]^{m\times m}$ and $\ov{\Ab}_{\mathrm{inter}} \in [0,1]^{m\times n}$ be the relaxed injected adjacency matrices. Let $B$ be a positive integer budget. The Euclidean Projection onto the set $\Ccal$, denoted by $\Pi_\Ccal(\ov{\Ab}_{\mathrm{inter}}, \ov{\Ab}_{\mathrm{intra}})$, can be computed by setting the entries corresponding to the $B$ largest values in each row of the concatenated matrix $[\ov{\Ab}_{\mathrm{inter}}, \ov{\Ab}_{\mathrm{intra}}]$ to $1$, and all others to $0$. 
\end{proposition}

\begin{proof}
First, we define the combined row-wise matrix for the $i$-th injected node as $\mathbf{r}_i := [\ov{\Ab}_{\mathrm{inter}}[i], \ov{\Ab}_{\mathrm{intra}}[i]] \in \mathbb{R}^{n+m}$. The constraint set $\Ccal$ requires that each row vector has exactly binary entries and sums to at most $B$.

The projection problem minimizes the squared Frobenius norm between the continuous combined matrix $\ov{\Ab}_{\mathrm{com}}$ and the discrete matrix $\Ab_{\mathrm{com}}$:
\begin{align*}
    \Ab^*_{\mathrm{com}} &= \mathop{\mathrm{argmin}}_{\Ab \in \Ccal} \sum_{i,j} (\ov{\Ab}_{ij} - \Ab_{ij})^2 \\
          &= \mathop{\mathrm{argmin}}_{\Ab \in \Ccal} \sum_{i,j} (\ov{\Ab}_{ij}^2 - 2\ov{\Ab}_{ij}\Ab_{ij} + \Ab_{ij}^2).
\end{align*}
Since $\Ab$ is binary ($\Ab_{ij} \in \{0, 1\}$), we have $\Ab_{ij}^2 = \Ab_{ij}$. Furthermore, assuming the optimal attack utilizes the full budget, every feasible row has exactly $B$ ones, implying $\sum_{j} \Ab_{ij}^2 = \sum_{j} \Ab_{ij} = B$. 
The term $\sum \ov{\Ab}_{ij}^2$ is constant with respect to $\Ab$. Thus, the minimization simplifies to:
\begin{align*}
    \Ab^*_{\mathrm{com}} &= \mathop{\mathrm{argmin}}_{\Ab \in \Ccal} \sum_{i} ( -2 \sum_{j} \ov{\Ab}_{ij} \Ab_{ij} + \text{const} ) \\
          &= \mathop{\mathrm{argmax}}_{\Ab \in \Ccal} \sum_{i} \sum_{j} \ov{\Ab}_{ij} \Ab_{ij}.
\end{align*}
The problem decomposes into $m$ independent row-wise maximization problems. For a specific injected node $i$, we seek to maximize the inner product $\sum_{j} \ov{\Ab}_{ij} \Ab_{ij}$ subject to $\sum_{j} \Ab_{ij} \le B$. By the rearrangement inequality, this sum is maximized when the non-zero entries of $\Ab_{ij}$ (which are all $1$) align with the largest coefficients in $\ov{\Ab}_{ij}$. 
Therefore, the optimal solution is to set $\Ab_{ij}=1$ if $\ov{\Ab}_{ij}$ is among the top-$B$ values in row $i$, and $0$ otherwise.
\end{proof}

\begin{figure*}[!ht]
    \centering
    \begin{subfigure}{0.25\linewidth}
        \centering
        \includegraphics[width=\linewidth]{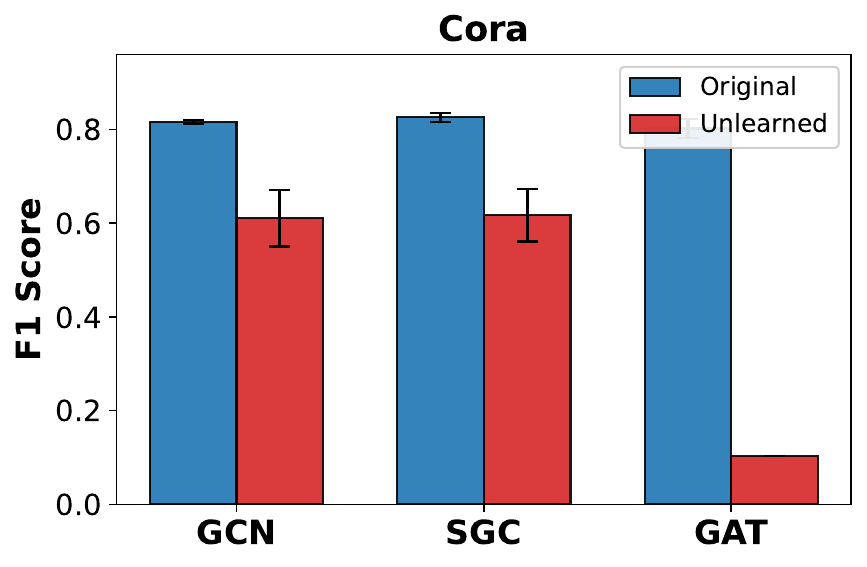}
        %\caption{Cora}
        \label{fig:model_cora}
    \end{subfigure}
    \hfill
    \begin{subfigure}{0.25\linewidth}
        \centering
        \includegraphics[width=\linewidth]{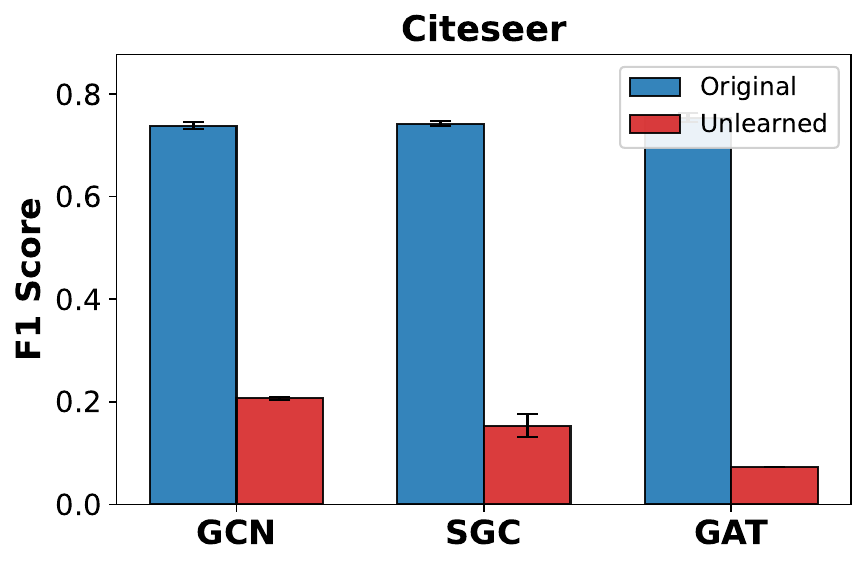}
        %\caption{Citeseer}
        \label{fig:model_citeseer}
    \end{subfigure}
    \hfill
    \begin{subfigure}{0.25\linewidth}
        \centering
        \includegraphics[width=\linewidth]{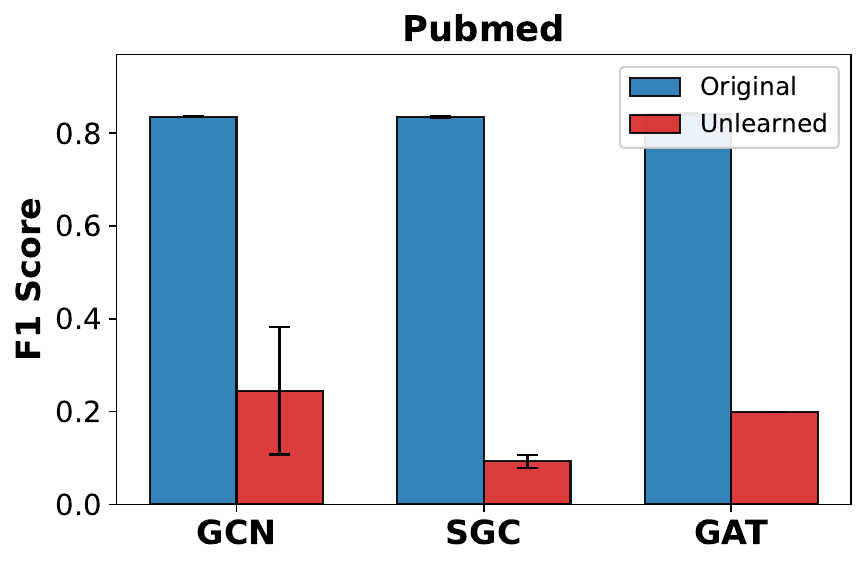}
        %\caption{Pubmed}
        \label{fig:model_pubmed}
    \end{subfigure}
    \vskip -0.3in
    \caption{\textbf{Attack Transferability across Victim Models.} Performance of \textbf{OptimAttack} (generated using a GCN surrogate) when applied to different victim model architectures: GCN, SGC, and GAT. The blue bars represent original accuracy (stealthiness), and red bars represent unlearned accuracy (damage).}
    %\vskip -0.1in
    \label{fig:model_transferability}
\end{figure*}

\begin{table*}[!ht]
\caption{\textbf{Main Comparison Results (with Standard Deviation).} Original accuracy, unlearned accuracy, and accuracy drop ($\Delta$Acc) across datasets. Best results are in \textbf{bold}, and second-best are \underline{underlined}. A clearer version without the standard deviation can be found in Table~\ref{tab:main_new}.}
\label{tab:main_new_full}
\centering
\resizebox{\linewidth}{!}{
\begin{tabular}{llccc|ccc|ccc|ccc}
\toprule[1.5pt]
\multirow{2}{*}{\textbf{\shortstack{Unlearn\\Method}}} & \multirow{2}{*}{\textbf{Attack}} 
  & \multicolumn{3}{c}{\textbf{Cora}} 
  & \multicolumn{3}{c}{\textbf{Citeseer}} 
  & \multicolumn{3}{c}{\textbf{Pubmed}} 
  & \multicolumn{3}{c}{\textbf{Flickr}} \\
\cmidrule(lr){3-5} \cmidrule(lr){6-8} \cmidrule(lr){9-11} \cmidrule(lr){12-14}
 & & Original$\uparrow$ & Unlearned$\downarrow$ & $\Delta$Acc$\downarrow$ 
   & Original$\uparrow$ & Unlearned$\downarrow$ & $\Delta$Acc$\downarrow$ 
   & Original$\uparrow$ & Unlearned$\downarrow$ & $\Delta$Acc$\downarrow$ 
   & Original$\uparrow$ & Unlearned$\downarrow$ & $\Delta$Acc$\downarrow$ \\
\midrule[1pt]
% ==================== GIF ====================
\multirow{10}{*}{\textbf{GIF}}
 & NoAttack & $0.8502 \pm 0.0106$ & $0.8458 \pm 0.0122$ & $0.0044 \pm 0.0162$ & $0.7357 \pm 0.0042$ & $0.7201 \pm 0.0044$ & $0.0156 \pm 0.0061$ & $0.8308 \pm 0.0021$ & $0.8180 \pm 0.0025$ & $0.0129 \pm 0.0033$ & $0.4355 \pm 0.0077$ & $0.4348 \pm 0.0064$ & $0.0007 \pm 0.0100$ \\
 \cmidrule(lr){2-14}
 & Random & $0.8317 \pm 0.0060$ & $0.8435 \pm 0.0038$ & $-0.0118 \pm 0.0071$ & $0.7201 \pm 0.0070$ & $0.7363 \pm 0.0075$ & $-0.0162 \pm 0.0102$ & $0.8267 \pm 0.0019$ & $0.7781 \pm 0.0743$ & $0.0486 \pm 0.0743$ & $0.4355 \pm 0.0078$ & $0.4345 \pm 0.0058$ & $0.0010 \pm 0.0097$ \\
 & Copy & $0.8494 \pm 0.0085$ & $0.8450 \pm 0.0077$ & $0.0044 \pm 0.0115$ & $0.7375 \pm 0.0024$ & $0.7411 \pm 0.0044$ & $-0.0036 \pm 0.0050$ & $0.8308 \pm 0.0012$ & $0.8302 \pm 0.0013$ & $0.0006 \pm 0.0018$ & $0.3648 \pm 0.1533$ & $0.3644 \pm 0.1531$ & $0.0004 \pm 0.2167$ \\
 & NewCopy & $0.8317 \pm 0.0127$ & $0.8435 \pm 0.0068$ & $-0.0118 \pm 0.0144$ & $0.7375 \pm 0.0041$ & $0.7376 \pm 0.0059$ & $-0.0001 \pm 0.0072$ & $0.8362 \pm 0.0024$ & $0.8316 \pm 0.0017$ & $0.0046 \pm 0.0030$ & $0.4438 \pm 0.0243$ & $0.4435 \pm 0.0238$ & $0.0003 \pm 0.0340$ \\
 & TestCopy & $0.8303 \pm 0.0070$ & $0.8295 \pm 0.0082$ & $0.0007 \pm 0.0108$ & $0.7153 \pm 0.0035$ & $0.7243 \pm 0.0055$ & $-0.0090 \pm 0.0065$ & $0.8440 \pm 0.0017$ & $0.7991 \pm 0.0708$ & $0.0449 \pm 0.0708$ & $0.4323 \pm 0.0013$ & $0.3961 \pm 0.0712$ & $0.0362 \pm 0.0712$ \\
 & TestLink & $0.7801 \pm 0.0103$ & $0.8332 \pm 0.0122$ & $-0.0531 \pm 0.0160$ & $0.7279 \pm 0.0049$ & $0.7628 \pm 0.0042$ & $-0.0348 \pm 0.0065$ & $0.8165 \pm 0.0029$ & $0.8329 \pm 0.0021$ & $-0.0163 \pm 0.0036$ & $0.4319 \pm 0.0006$ & $0.4317 \pm 0.0003$ & $0.0002 \pm 0.0007$ \\
 \cmidrule(lr){2-14}
 & SBA & $0.8259 \pm 0.0043$ & $0.8233 \pm 0.0049$ & $0.0026 \pm 0.0065$ & $0.7140 \pm 0.0044$ & $0.7134 \pm 0.0051$ & $0.0006 \pm 0.0067$ & $0.8262 \pm 0.0007$ & $0.8261 \pm 0.0004$ & $0.0001 \pm 0.0008$ & $0.4339 \pm 0.0007$ & $0.4220 \pm 0.0009$ & $0.0119 \pm 0.0011$ \\
 & UGBA & $0.8311 \pm 0.0046$ & $0.7091 \pm 0.0169$ & $0.1220 \pm 0.0175$ & $0.7173 \pm 0.0063$ & $0.7158 \pm 0.0067$ & $0.0015 \pm 0.0092$ & $0.8247 \pm 0.0009$ & $0.8245 \pm 0.0011$ & $0.0002 \pm 0.0014$ & $0.4208 \pm 0.0024$ & $0.4167 \pm 0.0028$ & $0.0041 \pm 0.0034$ \\
 & TDGIA & $0.7538 \pm 0.0089$ & $0.7231 \pm 0.0097$ & $0.0307 \pm 0.0132$ & $0.6917 \pm 0.0071$ & $0.6556 \pm 0.0098$ & $0.0361 \pm 0.0121$ & $0.8181 \pm 0.0059$ & $0.5791 \pm 0.1254$ & $0.2390 \pm 0.1255$ & $0.4057 \pm 0.0008$ & $0.4035 \pm 0.0013$ & $0.0022 \pm 0.0015$ \\
 \cmidrule(lr){2-14}
 & \textbf{OptimAttack} & $0.8162 \pm 0.0043$ & $0.6103 \pm 0.0600$ & $0.2059 \pm 0.0601$ & $0.7381 \pm 0.0064$ & $0.2066 \pm 0.0029$ & $0.5315 \pm 0.0071$ & $0.8351 \pm 0.0009$ & $0.2443 \pm 0.1372$ & $0.5908 \pm 0.1372$ & $0.4320 \pm 0.0008$ & $0.2860 \pm 0.1384$ & $0.1460 \pm 0.1384$ \\
\midrule[1pt]

% ==================== CEU ====================
\multirow{10}{*}{\textbf{CEU}}
 & NoAttack & $0.8214 \pm 0.0060$ & $0.8192 \pm 0.0066$ & $0.0022 \pm 0.0089$ & $0.7243 \pm 0.0044$ & $0.7075 \pm 0.0096$ & $0.0168 \pm 0.0106$ & $0.8084 \pm 0.0051$ & $0.7956 \pm 0.0053$ & $0.0128 \pm 0.0073$ & $0.4440 \pm 0.0249$ & $0.4441 \pm 0.0251$ & $-0.0001 \pm 0.0353$ \\
 \cmidrule(lr){2-14}
 & Random & $0.8103 \pm 0.0111$ & $0.8244 \pm 0.0095$ & $-0.0140 \pm 0.0146$ & $0.7195 \pm 0.0110$ & $0.7321 \pm 0.0108$ & $-0.0126 \pm 0.0154$ & $0.7951 \pm 0.0121$ & $0.3849 \pm 0.1536$ & $0.4102 \pm 0.1541$ & $0.4318 \pm 0.0004$ & $0.4316 \pm 8.9636$ & $0.0002 \pm 8.9636$ \\
 & Copy & $0.8221 \pm 0.0108$ & $0.8220 \pm 0.0135$ & $0.0001 \pm 0.0173$ & $0.7195 \pm 0.0082$ & $0.7159 \pm 0.0108$ & $0.0036 \pm 0.0136$ & $0.8049 \pm 0.0056$ & $0.7467 \pm 0.0836$ & $0.0582 \pm 0.0838$ & $0.4432 \pm 0.0232$ & $0.3987 \pm 0.0657$ & $0.0444 \pm 0.0697$ \\
 & NewCopy & $0.8089 \pm 0.0079$ & $0.8229 \pm 0.0099$ & $-0.0140 \pm 0.0127$ & $0.7195 \pm 0.0094$ & $0.7207 \pm 0.0076$ & $-0.0012 \pm 0.0121$ & $0.8041 \pm 0.0096$ & $0.7831 \pm 0.0246$ & $0.0210 \pm 0.0264$ & $0.4317 \pm 0.0003$ & $0.4316 \pm 0.0000$ & $0.0001 \pm 0.0003$ \\
 & TestCopy & $0.8155 \pm 0.0074$ & $0.8177 \pm 0.0050$ & $-0.0022 \pm 0.0089$ & $0.7207 \pm 0.0089$ & $0.7205 \pm 0.0087$ & $0.0002 \pm 0.0125$ & $0.8121 \pm 0.0110$ & $0.3869 \pm 0.1584$ & $0.4252 \pm 0.1588$ & $0.4420 \pm 0.0197$ & $0.4426 \pm 0.0216$ & $-0.0006 \pm 0.0293$ \\
 & TestLink & $0.7498 \pm 0.0120$ & $0.8185 \pm 0.0059$ & $-0.0686 \pm 0.0134$ & $0.6967 \pm 0.0127$ & $0.7411 \pm 0.0094$ & $-0.0444 \pm 0.0158$ & $0.7438 \pm 0.0238$ & $0.3274 \pm 0.1582$ & $0.4164 \pm 0.1599$ & $0.4322 \pm 0.0011$ & $0.4321 \pm 0.0010$ & $0.0000 \pm 0.0015$ \\
 \cmidrule(lr){2-14}
 & SBA & $0.8037 \pm 0.0054$ & $0.7993 \pm 0.0078$ & $0.0044 \pm 0.0095$ & $0.6941 \pm 0.0023$ & $0.6947 \pm 0.0039$ & $-0.0006 \pm 0.0045$ & $0.7985 \pm 0.0118$ & $0.5682 \pm 0.2948$ & $0.2303 \pm 0.2950$ & $0.4232 \pm 6.7227$ & $0.4230 \pm 0.0007$ & $0.0002 \pm 6.7227$ \\
 & UGBA & $0.8085 \pm 0.0083$ & $0.6909 \pm 0.0090$ & $0.1176 \pm 0.0123$ & $0.6926 \pm 0.0052$ & $0.6926 \pm 0.0048$ & $-0.0000 \pm 0.0071$ & $0.7877 \pm 0.0161$ & $0.6562 \pm 0.2311$ & $0.1315 \pm 0.2316$ & $0.4219 \pm 0.0024$ & $0.4218 \pm 0.0024$ & $0.0001 \pm 0.0034$ \\
 & TDGIA & $0.5309 \pm 0.0112$ & $0.5279 \pm 0.0152$ & $0.0030 \pm 0.0189$ & $0.5158 \pm 0.0181$ & $0.4857 \pm 0.0156$ & $0.0301 \pm 0.0239$ & $0.6116 \pm 0.0305$ & $0.6229 \pm 0.1205$ & $-0.0114 \pm 0.1243$ & $0.3912 \pm 0.0013$ & $0.3850 \pm 0.00023$ & $0.0062 \pm 0.0026$ \\
 \cmidrule(lr){2-14}
 & \textbf{OptimAttack} & $0.7904 \pm 0.0075$ & $0.5801 \pm 0.0376$ & $0.2103 \pm 0.0383$ & $0.7267 \pm 0.0102$ & $0.2276 \pm 0.0311$ & $0.4991 \pm 0.0327$ & $0.8102 \pm 0.0075$ & $0.2146 \pm 0.0228$ & $0.5956 \pm 0.0240$ & $0.4322 \pm 0.0002$ & $0.2858 \pm 0.1237$ & $0.1464 \pm 0.1237$ \\
\midrule[1pt]

% ==================== GA ====================
\multirow{10}{*}{\textbf{GA}}
 & NoAttack & $0.8443 \pm 0.0120$ & $0.8177 \pm 0.0113$ & $0.0266 \pm 0.0165$ & $0.7417 \pm 0.0074$ & $0.6805 \pm 0.0090$ & $0.0612 \pm 0.0117$ & $0.8320 \pm 0.0014$ & $0.8158 \pm 0.0023$ & $0.0162 \pm 0.0027$ & $0.4461 \pm 0.0178$ & $0.4336 \pm 0.0028$ & $0.0125 \pm 0.0180$ \\
 \cmidrule(lr){2-14}
 & Random & $0.8295 \pm 0.0072$ & $0.5166 \pm 0.0337$ & $0.3129 \pm 0.0344$ & $0.7171 \pm 0.0072$ & $0.4613 \pm 0.1023$ & $0.2559 \pm 0.1025$ & $0.8279 \pm 0.0011$ & $0.7041 \pm 0.0177$ & $0.1238 \pm 0.0177$ & $0.4433 \pm 0.0233$ & $0.4317 \pm 0.0002$ & $0.0116 \pm 0.0233$ \\
 & Copy & $0.8524 \pm 0.0040$ & $0.8258 \pm 0.0097$ & $0.0266 \pm 0.0105$ & $0.7321 \pm 0.0048$ & $0.7177 \pm 0.0112$ & $0.0144 \pm 0.0122$ & $0.8304 \pm 0.0014$ & $0.8291 \pm 0.0020$ & $0.0013 \pm 0.0025$ & $0.2125 \pm 0.1763$ & $0.4331 \pm 0.0029$ & $-0.2206 \pm 0.1763$ \\
 & NewCopy & $0.8325 \pm 0.0068$ & $0.8635 \pm 0.0084$ & $-0.0310 \pm 0.0108$ & $0.7393 \pm 0.0040$ & $0.6859 \pm 0.0052$ & $0.0535 \pm 0.0066$ & $0.8371 \pm 0.0012$ & $0.8334 \pm 0.0023$ & $0.0038 \pm 0.0026$ & $0.3615 \pm 0.1502$ & $0.4298 \pm 0.0042$ & $-0.0683 \pm 0.1502$ \\
 & TestCopy & $0.8273 \pm 0.0105$ & $0.7845 \pm 0.0289$ & $0.0428 \pm 0.0308$ & $0.7237 \pm 0.0038$ & $0.7117 \pm 0.0042$ & $0.0120 \pm 0.0057$ & $0.8455 \pm 0.0021$ & $0.8092 \pm 0.0053$ & $0.0363 \pm 0.0057$ & $0.4412 \pm 0.0192$ & $0.4014 \pm 0.0740$ & $0.0398 \pm 0.0765$ \\
 & TestLink & $0.7786 \pm 0.0126$ & $0.8391 \pm 0.0030$ & $-0.0605 \pm 0.0129$ & $0.7243 \pm 0.0077$ & $0.7435 \pm 0.0086$ & $-0.0192 \pm 0.0116$ & $0.8176 \pm 0.0040$ & $0.8338 \pm 0.0027$ & $-0.0161 \pm 0.0048$ & $0.4054 \pm 0.0780$ & $0.4183 \pm 0.0859$ & $-0.0128 \pm 0.1160$ \\
 \cmidrule(lr){2-14}
 & SBA & $0.8288 \pm 0.0036$ & $0.8348 \pm 0.0043$ & $-0.0059 \pm 0.0056$ & $0.7155 \pm 0.0042$ & $0.6845 \pm 0.0099$ & $0.0310 \pm 0.0108$ & $0.8265 \pm 0.0014$ & $0.8271 \pm 0.0024$ & $-0.0006 \pm 0.0028$ & $0.4339 \pm 0.0007$ & $0.4220 \pm 0.0009$ & $0.0119 \pm 0.0011$ \\
 & UGBA & $0.8322 \pm 0.0014$ & $0.7767 \pm 0.0061$ & $0.0555 \pm 0.0063$ & $0.7152 \pm 0.0036$ & $0.4262 \pm 0.1090$ & $0.2890 \pm 0.1090$ & $0.8252 \pm 0.0009$ & $0.8285 \pm 0.0017$ & $-0.0033 \pm 0.0020$ & $0.4208 \pm 0.0024$ & $0.4167 \pm 0.0028$ & $0.0041 \pm 0.0034$ \\
 & TDGIA & $0.7397 \pm 0.0156$ & $0.7512 \pm 0.0146$ & $-0.0115 \pm 0.0213$ & $0.6911 \pm 0.0072$ & $0.6830 \pm 0.0073$ & $0.0081 \pm 0.0102$ & $0.7924 \pm 0.0518$ & $0.7905 \pm 0.0490$ & $0.0019 \pm 0.0713$ & $0.4055 \pm 0.0007$ & $0.4002 \pm 0.0014$ & $0.0053 \pm 0.0016$ \\
 \cmidrule(lr){2-14}
 & \textbf{OptimAttack} & $0.8258 \pm 0.0097$ & $0.4790 \pm 0.0238$ & $0.3469 \pm 0.0258$ & $0.7345 \pm 0.0065$ & $0.4138 \pm 0.0994$ & $0.3207 \pm 0.0996$ & $0.8279 \pm 0.0018$ & $0.6798 \pm 0.0493$ & $0.1481 \pm 0.0493$ & $0.4318 \pm 0.0007$ & $0.3012 \pm 0.1029$ & $0.1306 \pm 0.1029$ \\
\bottomrule[1.5pt]
\end{tabular}
}
\end{table*}

\appsection{Implementation Details}{sec:impl_detail}

{\bf Intuitive Baselines.} We compare our optimization-based attack with several intuitive baseline attacks, described as follows:
\begin{itemize}[leftmargin=0.05\linewidth]
    \item {\bf Random:} Randomly inject $m$ nodes into the training graph. Each injected node’s feature vector is sampled from a Gaussian distribution matching the mean and standard deviation of the original node features. The edges and labels of these injected nodes are assigned uniformly at random, where each node has exactly $B$ edges.
    \item {\bf Copy:} Randomly select $m$ labeled nodes from the training graph and duplicate them with the same features and labels. If a copied node exceeds the edge budget $B$, a subset of its edges is randomly dropped to satisfy the budget constraint.
    \item {\bf NewCopy:} Rank all labeled training nodes by the number of edges they have to unlabeled nodes, and inject the top-$m$ nodes with copied features and labels. When randomly selecting edges for the injected nodes, priority is given to connections with unlabeled nodes.
    \item {\bf TestCopy:} Similar to LabeledCopy, but $m$ nodes are selected from the unlabeled set. The duplicated nodes retain the same features, while their labels are randomly assigned.
    \item {\bf TestLink:} Randomly select $m$ nodes from the test set (unlabeled during training) to copy their features, but assign them random labels. Unlike TestCopy, which duplicates existing edges, this method connects each injected node to $B$ randomly selected test nodes, explicitly creating direct links between the injected triggers and the victim population.
\end{itemize}

{\bf Implementation Details or OptimAttack.} For the injection budget, we set the number of injected nodes to $5\%$ of the training graph size, with an edge budget of $B=5$ per node. We perform optimization for $T=200$ steps on Cora, Citeseer, and Pubmed, and $T=10$ steps on Flickr. The learning rates are set to $\eta_a = 0.5$ for adjacency matrix optimization and $\eta_x = 5 \times 10^{-4}$ for feature optimization. On the larger Flickr dataset, we apply our scalable approximation strategy: instead of optimizing connections to all test nodes, we sample $10\%$ of the test set and select 3 of their 1-hop neighbors as the candidate connection set. 
The surrogate model $\ov{f}$ follows the same architecture as the victim model (a two-layer GCN). We initialize the optimization using a baseline attack chosen via validation performance: we use Copy initialization for GIF and CEU, and Random initialization for GA. The unlearning process on the surrogate model uses Gradient Ascent (GA) with a learning rate of $0.1$ and a coefficient $\gamma=1.0$.

\appsection{Time Complexity}{sec:time}

{\bf Complexity of Full Optimization.}
Consider the optimization process in Algorithm~\ref{alg:pgd} on a graph with $n$ nodes, $E$ edges, and feature dimension $d$. The surrogate model is a standard GNN (e.g., GCN) with $L$ layers.
In the naive formulation, the optimization variables include the dense matrix $\ov{\Ab}_{\mathrm{inter}} \in \mathbb{R}^{m \times n}$ and $\ov{\Ab}_{\mathrm{intra}} \in \mathbb{R}^{m \times m}$.
Consequently, the computation of gradients $\nabla_{\ov{\Ab}_{\mathrm{inter}}} \Lcal_{\mathrm{atk}}$ necessitates backpropagating through a dense adjacency structure, resulting in a time complexity of $O(L(E + mn)d + (n+m)d^2)$ per attack step.
The term $O(Lmnd)$ for message passing on the injected part of adjacency matrix can be large on large scale graphs (e.g., $n > 5\times 10^4$) and with more attack node budget $m$. 

{\bf Reduced Complexity.}
Let $n_{\mathrm{sub}} = |\mathcal{S}_{\mathrm{cand}}|$ denote the size of the relevant subgraph. By restricting optimization to this subset, the number of active edge variables reduces from $mn$ to $m \cdot n_{\mathrm{sub}}$. Since $n_{\mathrm{sub}} \ll n$ in large sparse graphs, the per-step time complexity becomes $O(L(E + m \cdot n_{\mathrm{sub}})d + (n+m)d^2)$.
This enables our method to efficiently scale to datasets with millions of nodes.

\appsection{Additional Experiment Results}{sec:more_results}

{\bf Victim Model Transferrability.} We evaluate whether our unlearning corruption attack, optimized using a surrogate 2-layer GCN, can transfer to other GNN architectures. Figure~\ref{fig:model_transferability} reports the attack performance on three victim models: GCN, SGC, and Graph Attention Networks (GAT). We observe that the attack remains highly effective across all architectures, demonstrating strong transferability. Specifically, SGC exhibits vulnerability levels comparable to GCN, likely due to their shared spectral filtering roots. Notably, GAT appears even more susceptible to the attack, suffering the most severe performance collapse after unlearning (e.g., on Cora, GAT accuracy drops to near random). This suggests that the corruption patterns identified by our GCN-based surrogate exploit fundamental vulnerabilities in graph message-passing mechanisms rather than overfitting to a specific architecture.

{\bf Standard Deviation of Main Results.} In Table~\ref{tab:main_new_full}, we provide a full version of our comparison results in Table~\ref{tab:main_new}, which includes the standard deviation of all the experiments.

\ifdefined\isarxiv

\clearpage
\bibliographystyle{ACM-Reference-Format}
\bibliography{ref}

\else

\fi 

\end{document}